\documentclass{article} 
\usepackage[table]{xcolor}
\usepackage{colortbl}

\usepackage{iclr2026_conference,times}

\usepackage{amsmath,amsfonts,bm}

\def\eqref#1{equation~\ref{#1}}

\def\1{\bm{1}}

\DeclareMathAlphabet{\mathsfit}{\encodingdefault}{\sfdefault}{m}{sl}
\SetMathAlphabet{\mathsfit}{bold}{\encodingdefault}{\sfdefault}{bx}{n}

\usepackage{hyperref}
\usepackage{url}

\usepackage{times}  
\usepackage{helvet}  
\usepackage{courier}  
\iclrfinalcopy 
\usepackage{fancyhdr}
\fancypagestyle{preprint}{%
  \fancyhf{}%
  \lhead{\footnotesize Preprint}%
  \renewcommand{\headrulewidth}{0.4pt}
}

\usepackage{graphicx}      
\usepackage{subcaption}    

\usepackage{booktabs}

\usepackage{caption} 
\usepackage{algorithm}
\usepackage{algorithmic}
\usepackage{amsmath}
\usepackage{amsmath,amssymb,amsfonts}

\usepackage{pgfplots}
\usepackage{amsmath}
\usepackage{comment}
\usepackage{wrapfig}   

\usepackage{amsthm}

\newtheorem{assumption}{Assumption}
\newtheorem{theorem}{Theorem}

\newtheorem{lemma}{Lemma}

\definecolor{LightRed}{RGB}{255,182,193} 
\definecolor{LightBlue}{RGB}{173, 216, 230}

\colorlet{shadecolor}{LightBlue}
\colorlet{shadecolor}{LightRed}

\usepackage{wrapfig}
\usepackage{framed}  
\usepackage{multirow} 
\usepackage{enumitem}
\pgfplotsset{compat=1.18}
\usepackage{wrapfig}
\usepackage{graphicx}
\usepackage{newfloat}
\usepackage{listings}

\title{FedMuon: Accelerating Federated Learning with Matrix Orthogonalization}

\author{Junkang Liu$^1$, Fanhua Shang$^1$, Junchao Zhou$^1$, Hongying Liu$^1$, Yuanyuan Liu$^2$, Jin Liu$^2$  \\
$^1$Tianjin University, $^2$Xidian University, \\
}
\begin{document}

\maketitle

\begin{abstract}
The core bottleneck of Federated Learning (FL) lies in the communication rounds. That is, how to achieve more effective local updates is crucial for reducing communication rounds. Existing FL methods still primarily use element-wise  local optimizers (Adam/SGD), neglecting the geometric structure of the weight matrices. This often leads to the amplification of pathological directions in the weights during local updates, leading deterioration in the condition number and slow convergence. Therefore, we introduce the Muon optimizer in
local (named \texttt{Local Muon}), which has matrix orthogonalization to optimize matrix-structured parameters. 
Experimental results show that, in IID setting, \texttt{Local Muon} significantly accelerates the convergence of FL and reduces communication rounds compared to Local SGD and Local AdamW. However, in non-IID setting, independent matrix orthogonalization based on the local distributions of each client induces strong client drift. Applying Muon in non-IID FL poses significant challenges: (1) client preconditioner leading to client drift; (2) moment reinitialization.  To address these challenges, we propose a novel  \underline{Fed}erated \underline{Muon} optimizer (\texttt{FedMuon}), which incorporates two key techniques: (1) momentum aggregation, where clients use the aggregated momentum for local initialization; (2) local-global alignment, where the local gradients are aligned with the global update direction to significantly reduce client drift. Theoretically, we prove that \texttt{FedMuon} achieves a linear speedup convergence rate of $\mathcal{O}(\sqrt{(L \Delta \sigma_l^2)/(S K R)}+(L \Delta)/R)$ without the heterogeneity assumption, where $S$ is the number of participating clients per round, $K$ is the number of local iterations, and $R$ is the total number of communication rounds. 
 Empirically, we validate the effectiveness of \texttt{FedMuon} on language and vision models. Compared to several baselines, \texttt{FedMuon} significantly reduces communication rounds and improves test accuracy. The code is available in 
 \url{ https://github.com/junkangLiu0/FedMuon}
\end{abstract}

\section{Introduction}

With the rapid growth of data and rising concerns over user privacy, traditional centralized training paradigms have become inadequate. Federated Learning (FL) \cite{mcmahan2017communication} offers a scalable and privacy-preserving framework that enables collaborative model training across decentralized clients without sharing raw data~\citep{liu2024fedbcgd}. As data becomes increasingly siloed, FL is a practical solution for large-scale distributed deep learning.
However, data heterogeneity and limited communication rounds create significant bottlenecks in FL. Recent studies reveal that the Hessian matrix in neural networks exhibits an approximate block-diagonal structure with several dense sub-blocks~\citep{collobert2004large,zhang2024adam}, as shown in Figure~\ref{babygpt_hessian_plot}. Understanding parameter matrix structures is crucial for effective federated aggregation, yet this perspective has been largely overlooked in the federated learning literature.
 Currently, when clients use element-wise optimizers (such as AdamW/SGD) for multi-step updates on their local data, the weight matrices may gradually become ill-conditioned (see  Figure \ref{fig:comlare}), causing the update directions to either cancel out or amplify after aggregation. As a result, in each communication round clients struggle to obtain effective updates, and the global model converges slowly.

Recent advancements in the Muon optimizer offer a novel solution to this challenge. The Muon optimizer \citep{jordan6muon} has recently demonstrated that orthogonal normalization of weight update matrices can significantly accelerate neural network training (see  Figure \ref{fig:SVD}). By conditioning the weight updates to produce consistent changes in the hidden states, orthogonal normalization updates lead to faster convergence, improved training stability, and better hyperparameter transferability across different model scales \citep{bernstein2024modular,large2024scalable,pethick2025training}. Moonshot AI \citep{liu2025muon} found that, when training a 16B model, Muon achieved nearly twice the computational efficiency compared to AdamW ~\citep{loshchilov2017fixing}. Similarly, Essential AI \citep{shah2025practical} observed significant improvements with Muon in large-batch training.  Both GLM~4.5 and K2 are trained with the Muon optimizer \citep{liu2025muon}.
These features suggest that using Muon for local training in FL (\texttt{Local Muon}) could accelerate local training and reduce communication rounds.

We have also validated the effectiveness of \texttt{Local Muon} in FL in IID setting. \texttt{Local Muon} significantly outperforms Local SGD  and Local AdamW  (see  Figure \ref{fig:iid}). \texttt{Local Muon} accelerates local convergence and reduces the number of communication rounds required to reach the same level of precision, with faster local loss decrease, smoother training curves, and faster global model convergence (see  Figure \ref{fig:iid}). 
\textit{However, in non-IID setting, although the local losses of each client still decrease rapidly, the global model after aggregation becomes significantly unstable or even fails to converge (see  Figure \ref{fig:iid}).} We identify the reasons why the Muon optimizer fails in the case of non-IID federated learning from two complementary perspectives.
\vspace{-1mm}
{\colorlet{shadecolor}{LightBlue}
\begin{snugshade}
	\begin{center}
		{\bf(Challenge 1)} \textit{\textbf{Client preconditioner leading to client drift}: In non-IID FL, Muon’s client-specific preconditioner scales gradients from local data distribution, causing misalignment in aggregation.}
	\end{center}
	\vspace{-0.2cm}
\end{snugshade}}
\vspace{-2mm}
{\colorlet{shadecolor}{LightRed}
\begin{snugshade}
	\begin{center}
		{\bf(Challenge 2)} \textit{\textbf{Moment reinitialization:} reinitializing the moment of Muon from scratch in every round hinders the convergence.}
	\end{center}
	\vspace{-0.2cm}
\end{snugshade}}
\vspace{-2mm}
These challenges motivate us to develop a novel \textbf{\underline{Fed}erated \underline{Muon}} optimizer, \texttt{FedMuon}, the first FL optimizer that explicitly accounts for the structure of update matrices. \texttt{FedMuon} addresses the impact of non-IID data through two key mechanisms:  (1) \textbf{local-global alignment}, where the current local gradients are aligned with the global update  to significantly reduce cross-client inconsistency; (2) \textbf{momentum aggregation}, where clients initialize using the aggregated momentum.

Theoretically, we prove that \texttt{FedMuon} achieves a linear speedup convergence rate of $\mathcal{O}(\sqrt{(L \Delta \sigma_l^2)/(S K R)}+(L \Delta)/R)$ without the heterogeneity assumption, where $S$ is the number of participating clients per round, $K$ is the number of local iterations, and $R$ is the total number of communication rounds. Due to  the local-global alignment, our convergence speed is unaffected by data heterogeneity. Empirical results on ViT ~\citep{dosovitskiy2020image} and LLMs ~\citep{liu2019roberta} confirm that \texttt{FedMuon} improves test accuracy and reduces communication overhead compared to strong FL baselines.

\begin{figure}[tb]
	\centering    
	\begin{subfigure}[b]{0.16\textwidth}
		\includegraphics[width=\textwidth]{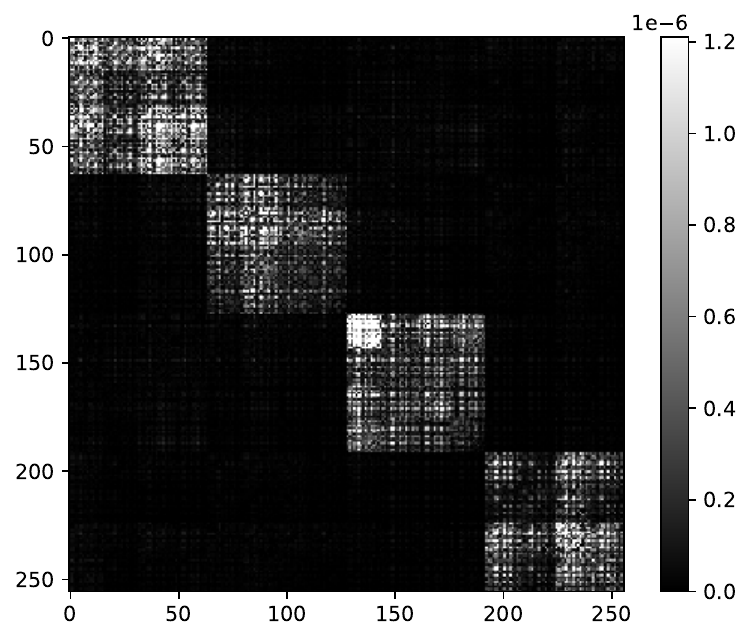}
		\caption{\texttt{query} }
	\end{subfigure}
	\begin{subfigure}[b]{0.155\textwidth}
		\includegraphics[width=\textwidth]{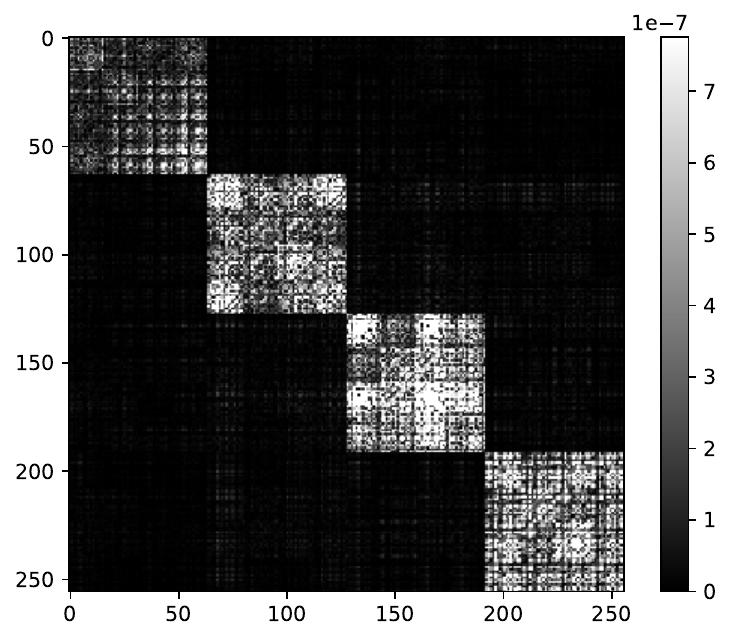}
		\caption{\texttt{key} }
	\end{subfigure}
	\begin{subfigure}[b]{0.16\textwidth}
		\includegraphics[width=\textwidth]{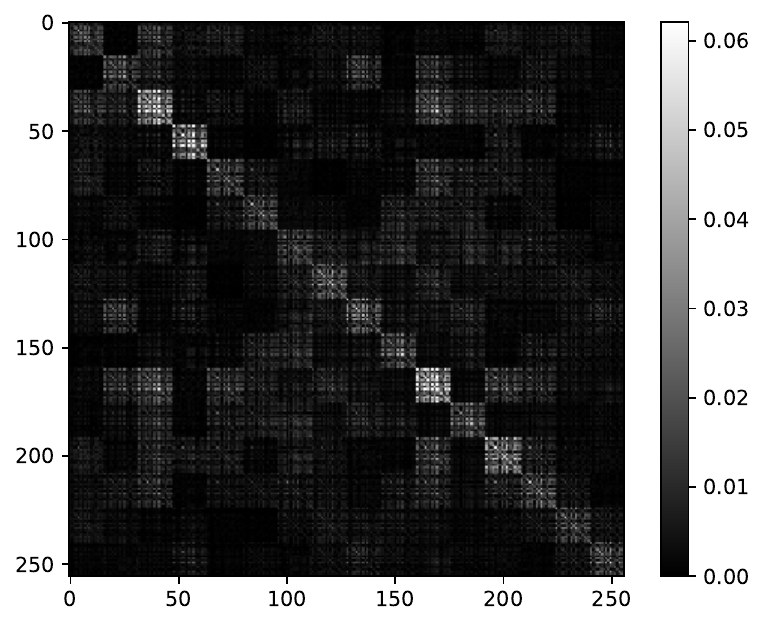}
		\caption{\texttt{value} }
	\end{subfigure}
	\begin{subfigure}[b]{0.16\textwidth}
		\includegraphics[width=\textwidth]{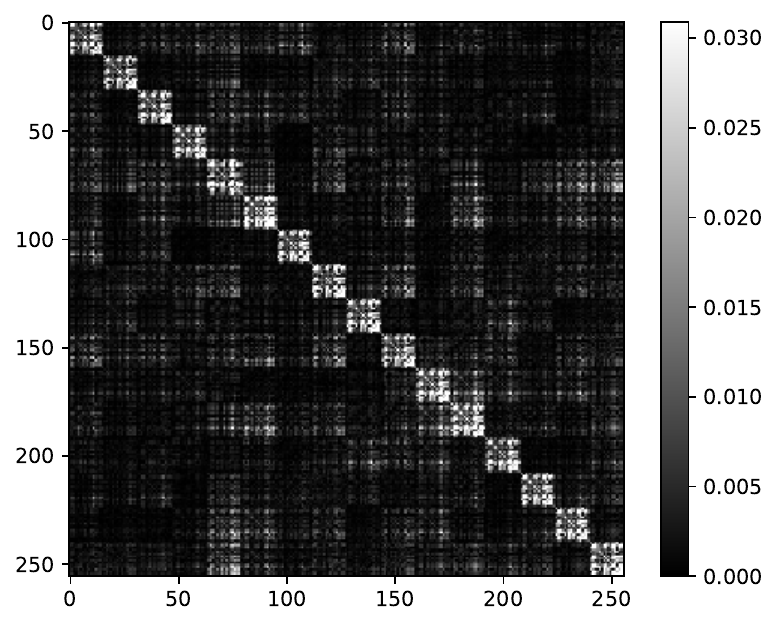}
		\caption{\texttt{attn.proj} }
	\end{subfigure}
	\begin{subfigure}[b]{0.16\textwidth}
		\includegraphics[width=\textwidth]{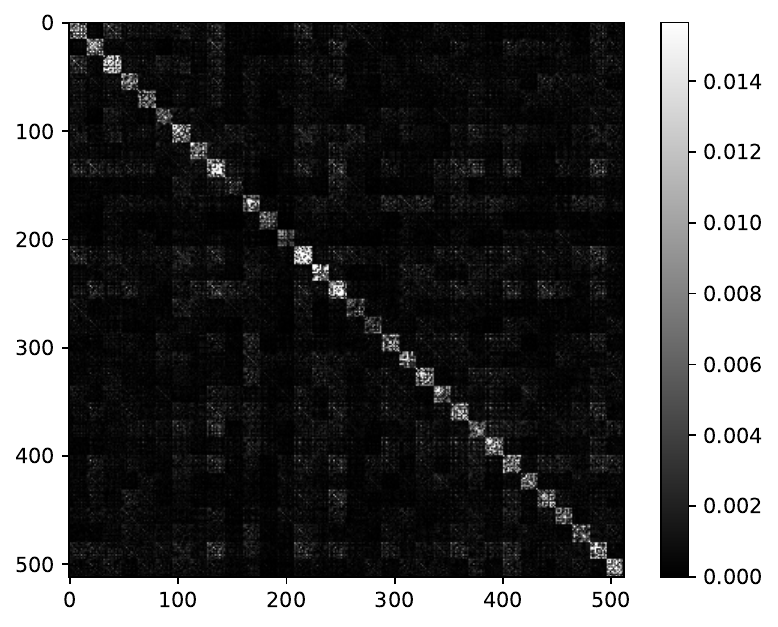}
		\caption{\texttt{mlp.fc\_1} }
	\end{subfigure}
	\begin{subfigure}[b]{0.17\textwidth}
		\includegraphics[width=\textwidth]{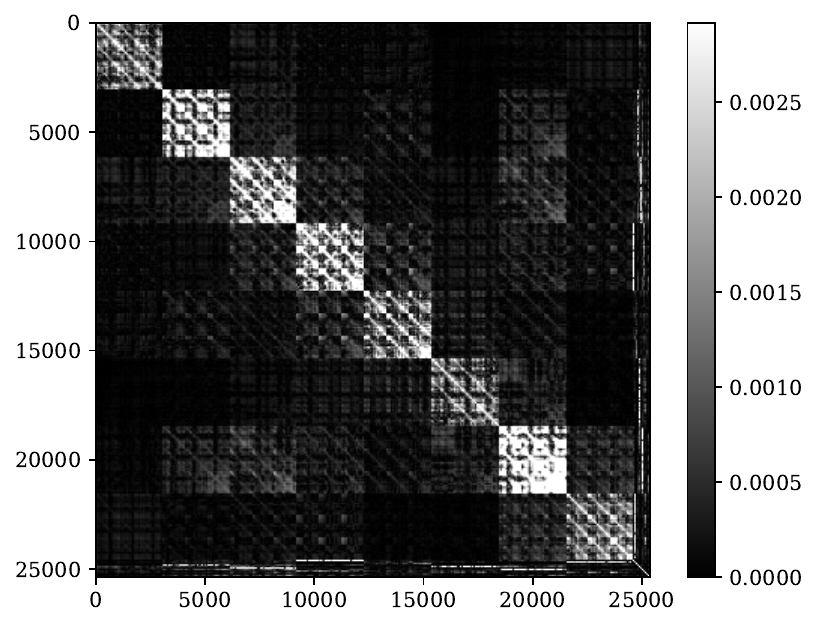}
		\caption{\texttt{MLP} }
	\end{subfigure}
	\vspace{-2mm}
	\caption{\small (a–f):Block-wise Hessian structure of Transformer parameters and MLP ~\citep{zhang2024adam}.}
	\label{babygpt_hessian_plot}
\end{figure}

\begin{figure*}[tb]
	\centering
	
	\begin{subfigure}[b]{0.47\textwidth}
		\centering
		\includegraphics[width=\linewidth]{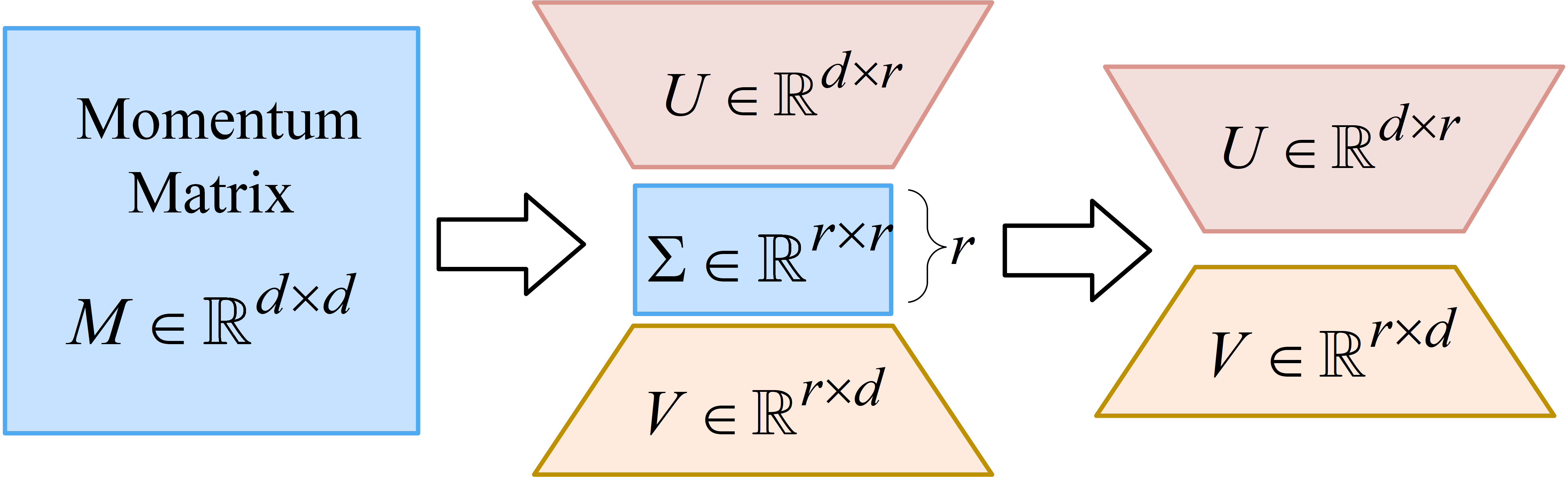}
		\caption{Matrix orthogonalizatio with SVD}
		\label{fig:sgd_grid_search:vit}
	\end{subfigure}
	\hfill
	\begin{subfigure}[b]{0.47\textwidth}
		\centering
		\includegraphics[width=\linewidth]{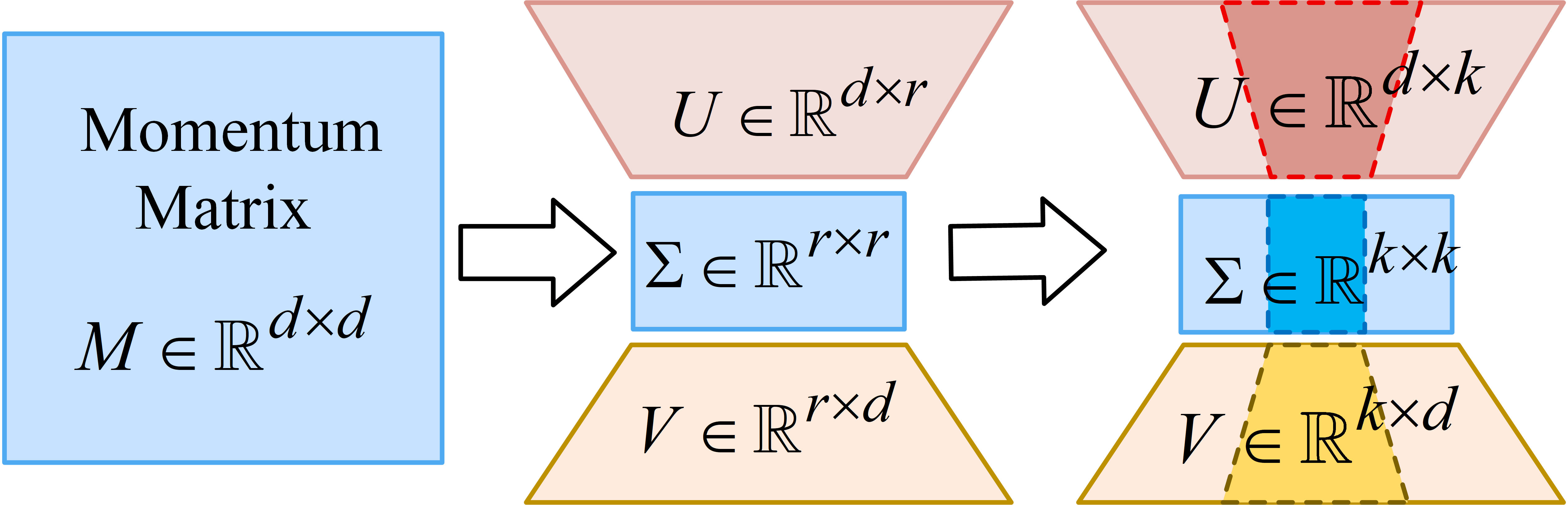}
		\caption{Matrix compression using SVD}
		\label{fig:sgd_grid_search:bert}
	\end{subfigure}
	\vspace{-2mm}
	\caption{\small 
		(a) shows SVD-based matrix orthogonalization; (b) applies SVD to the momentum matrix
		$M\!\in\!\mathbb{R}^{d\times d}$, i.e., $M \approx U\Sigma V^{\top}$, and keeps the top-$k$
		singular vectors to obtain $U\!\in\!\mathbb{R}^{d\times k}$ and $V\!\in\!\mathbb{R}^{k\times d}$ .} 
	\label{fig:SVD}
\end{figure*}

\begin{figure*}[tb]
	\centering
	\begin{subfigure}[b]{0.24\textwidth}
		\centering
		\includegraphics[width=\linewidth]{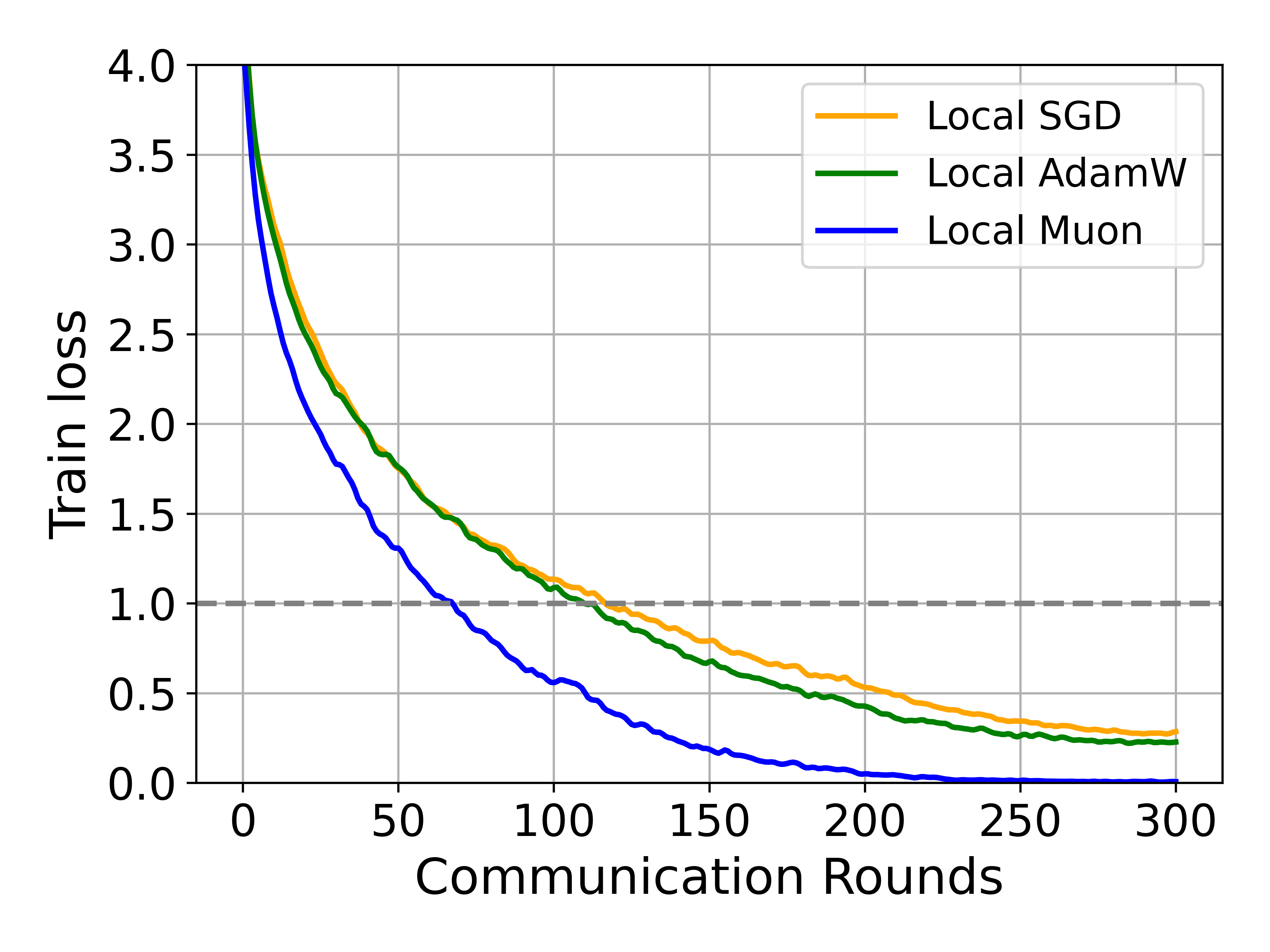}
		\caption{ResNet-18, IID}
		\label{fig:sgd_grid_search:vit}
	\end{subfigure}
	\begin{subfigure}[b]{0.24\textwidth}
		\centering
		\includegraphics[width=\linewidth]{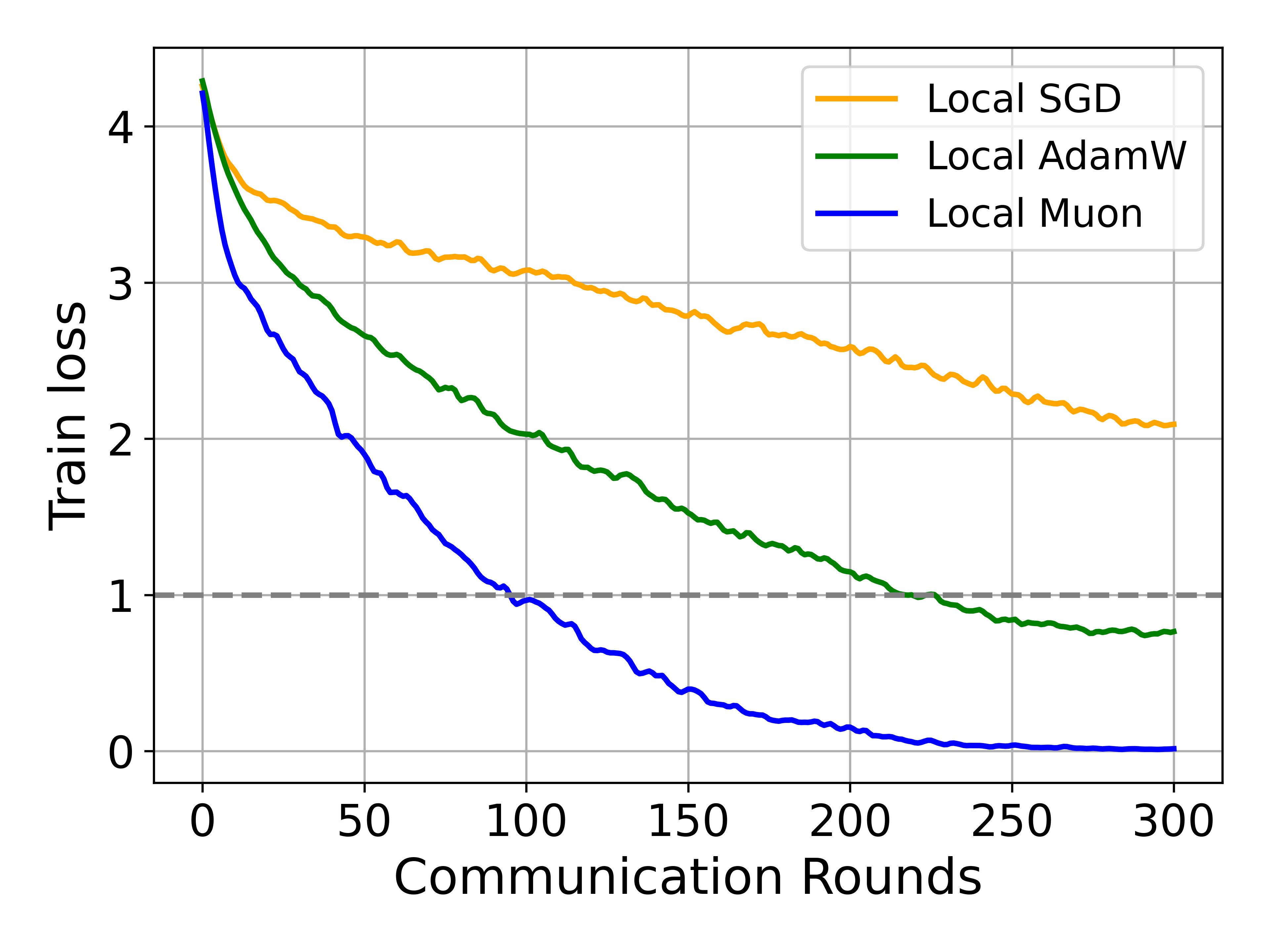}
		\caption{ViT-Tiny, IID}
		\label{fig:sgd_grid_search:gpt2}
	\end{subfigure}
	\begin{subfigure}[b]{0.24\textwidth}
		\centering
		\includegraphics[width=\linewidth]{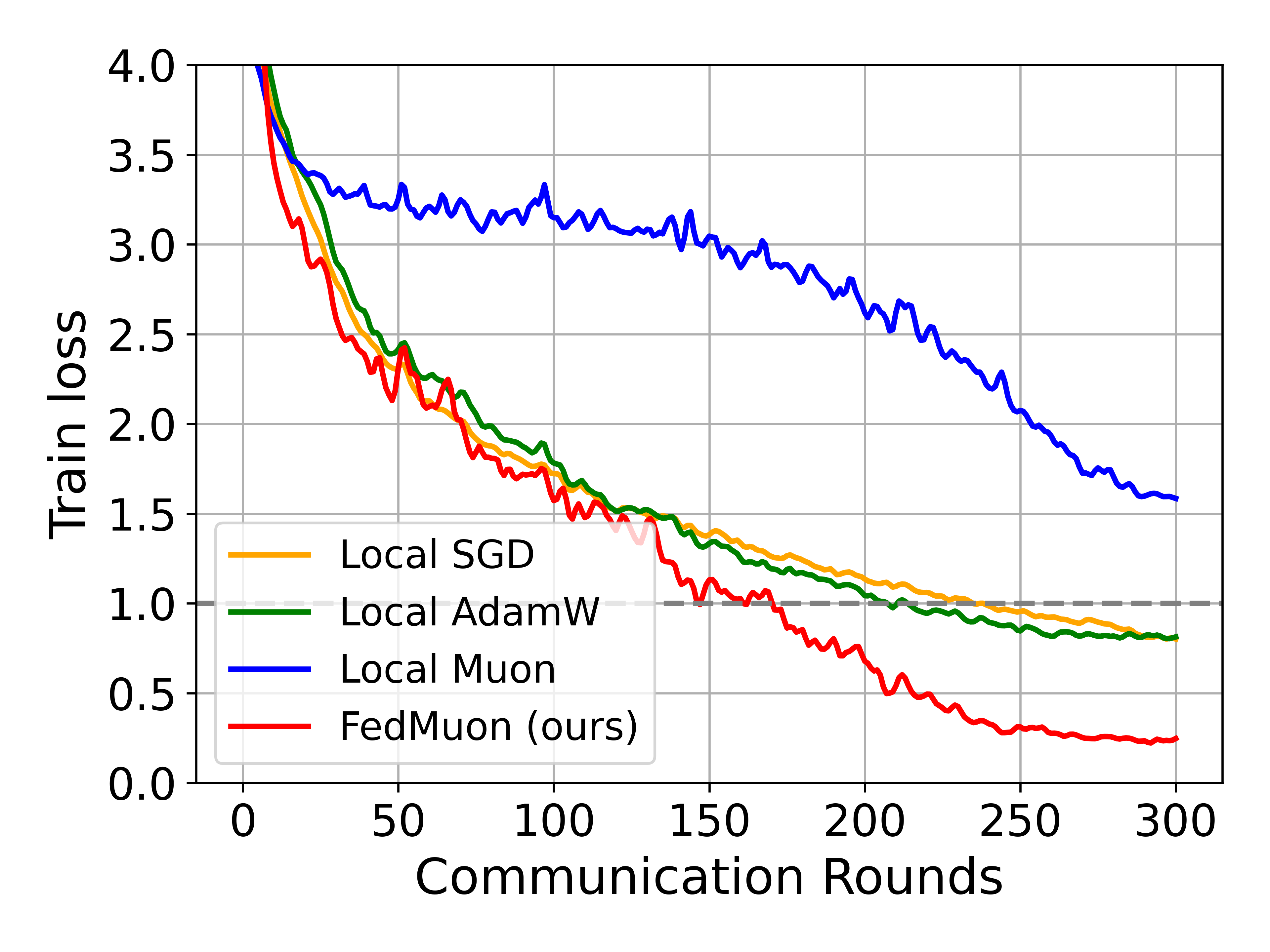}
		\caption{ResNet-18, non-IID}
		\label{fig:sgd_grid_search:bert}
	\end{subfigure}
		\begin{subfigure}[b]{0.24\textwidth}
		\centering
		\includegraphics[width=\linewidth]{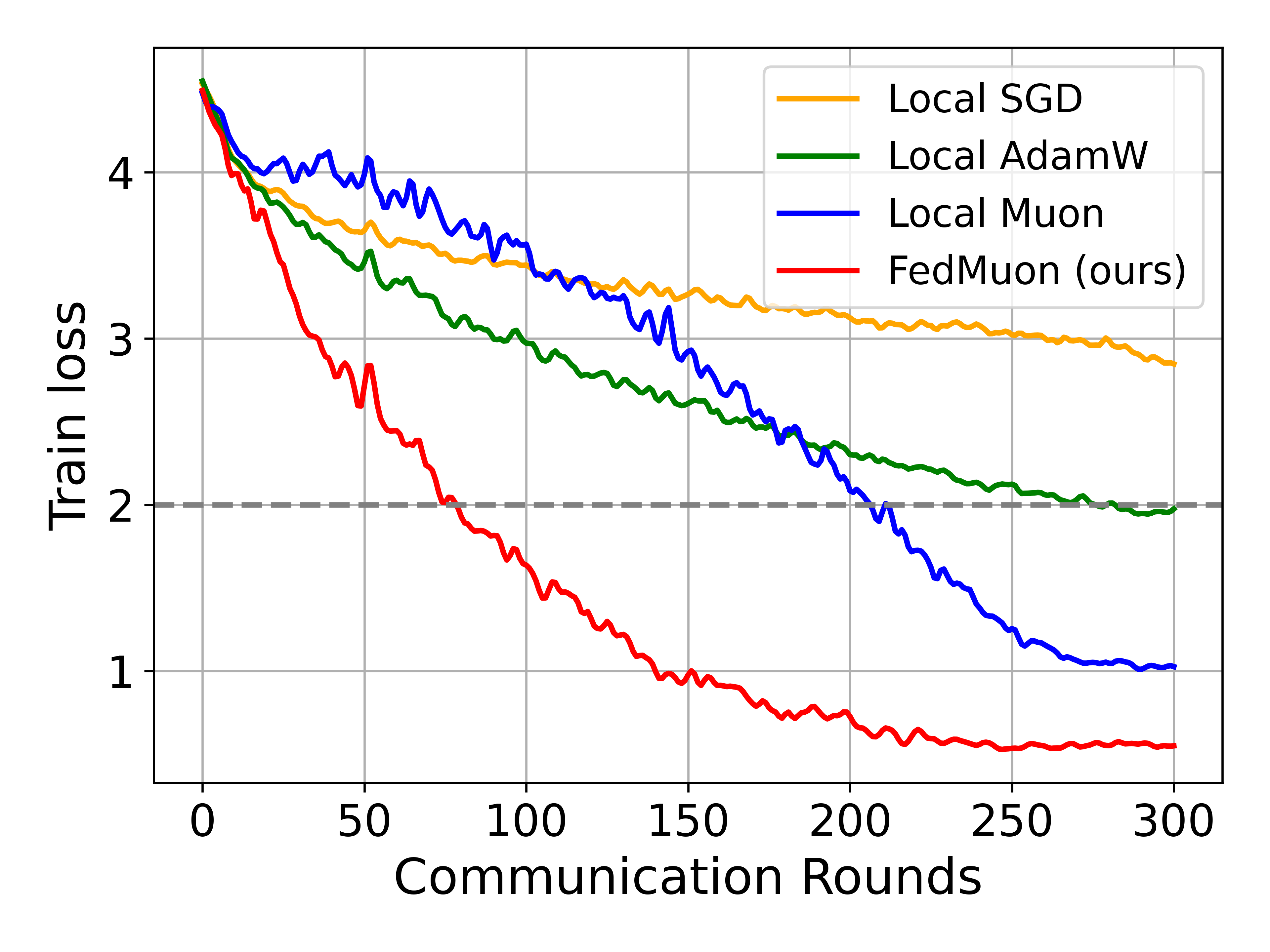}
		\caption{ViT-Tiny, non-IID}
		\label{fig:sgd_grid_search:bert}
	\end{subfigure}
	\vspace{-2mm}
	\caption{\small Performance of Local SGD, Local AdamW and \texttt{Local Muon}, we carefully tune the learning rate. }
	\label{fig:iid}
\end{figure*}

\begin{figure*}[tb]
	\centering
	\begin{subfigure}[b]{0.7\textwidth}
		\centering
		\includegraphics[width=\linewidth]{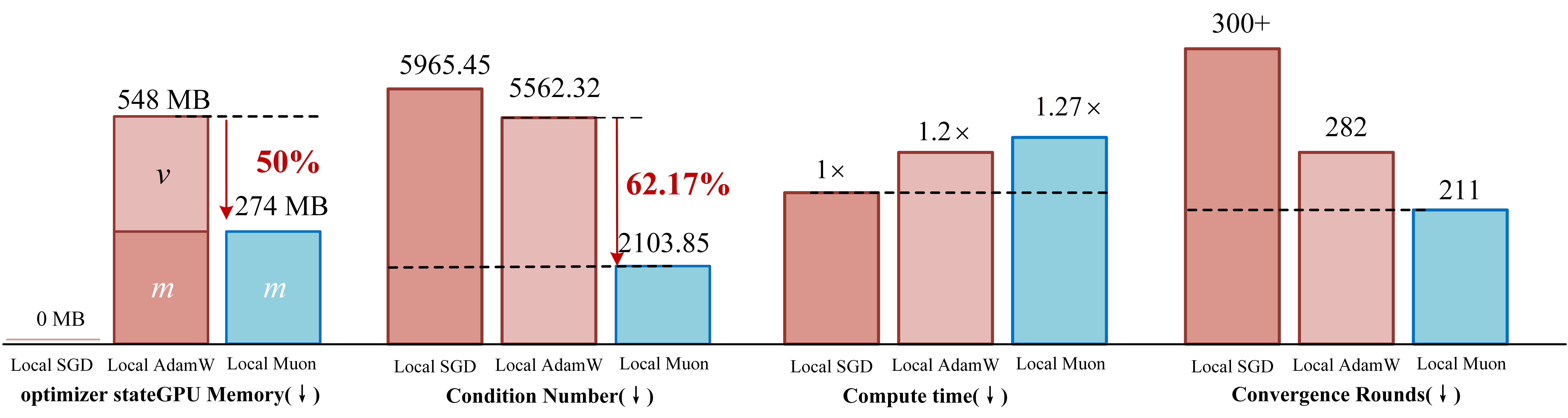}
		\caption{Comparison of local optimizers in FL.}
		\label{fig:sgd_grid_search:vit}
	\end{subfigure}
			\begin{subfigure}[b]{0.26\textwidth}
		\centering
		\includegraphics[width=\linewidth]{image/deit-cifar100-noniid.png}
		\caption{ViT-Tiny, non-IID}
		\label{fig:sgd_grid_search:bert}
	\end{subfigure}
	\vspace{-2mm}
	\caption{\small    
		(a) Analysis on ViT-Tiny with CIFAR-100, showing optimizer state memory, condition number, computation time, and convergence rounds. Local Muon achieves lower memory cost, lower the condition number, and faster convergence.
		(b) Training loss curves of ViT-Tiny under non-IID.  } 
	\label{fig:comlare}
\end{figure*}

\noindent \textbf{Our contributions} are summarized as follows:
\begin{itemize}[leftmargin=*]
	\item \textbf{Introducing Muon into Federated Learning.} We are the first to design a federated optimizer that explicitly considers the structure of parameter matrices, introducing the matrix orthogonalization method (i.e.,  \texttt{Muon}) into federated learning. Extensive experiments demonstrate its superiority. However, in highly non-IID settings, severe client drift arises. We analyze this issue from two perspectives: (1) \textbf{client preconditioner leading to client drift}, (2) \textbf{moment reinitialization}.
	\item \textbf{We propose \texttt{FedMuon}, a principled FL algorithm with Matrix Orthogonalization.} To address above challenges, \texttt{FedMuon} introduce the two mechanisms, \textbf{local-global alignment} and \textbf{momentum aggregation}. Inspired by the Hessian structure, we also design a communication-efficient aggregation strategy that communicates the SVD compression of momentum. 
	\item \textbf{Theoretical guarantees with improved convergence.} 
	\texttt{FedMuon} achieves a linear  convergence rate of 
	$\mathcal{O}(\sqrt{(L \Delta \sigma_l^2)/(S K R)}+(L \Delta)/R)$ without the widely used data  heterogeneity assumption. Due to  the local-global alignment, its convergence speed is unaffected by data heterogeneity. 
	
\end{itemize}

\section{Related Work}
\vspace{-2mm}
\textbf{$\bullet$ Optimizers in non-IID Federated Learning.} Data heterogeneity across clients is a fundamental challenge in FL. A range of algorithms have been proposed to mitigate the adverse effects of non-i.i.d. data distributions. For example, FedProx~\citep{li2018federated} introduces a proximal term to restrict local updates; SCAFFOLD~\citep{karimireddy2020scaffold} applies control variates to correct client drift; and FedCM~\citep{xu2021fedcm} leverages client momentum to stabilize updates. FedOpt~\citep{reddi2020adaptive} incorporates server-side adaptivity using Adam. More recently, \citet{sun2023efficient} proposed FedLADA to only aggregate the second-moment estimate of Adam to overcome client drift. 

\textbf{$\bullet$ Optimizers in Centralized Settings.}
Adaptive gradient methods have demonstrated superior empirical performance over SGD in centralized settings, particularly for deep neural networks. Pioneering works include Adagrad~\citep{duchi2011adaptive}, Adadelta~\citep{zeiler2012adadelta}, Adam~\citep{kingma2014adam}, AMSGrad~\citep{reddi2019convergence}, and AdamW~\citep{loshchilov2017fixing}. 
Other structured optimization methods. Although commonly used optimizers for training
deep neural networks, such as SGD, Adam, and AdamW, typically treat structured parameters (e.g., matrices) as flattened vectors, however, in recent
years, there has been growing interest in designing structured optimizers that explicitly leverage the
inherent structure of parameters. The Muon optimizer \citep{jordan6muon} has recently demonstrated that orthogonal normalization of weight update matrices can significantly accelerate neural network training.


\section{FL Problem Setup}
\vspace{-2mm}
FL aims to optimize model parameters with  local clients, i.e., minimizing the following problem:
\vspace{-2mm}
\begin{equation}
	f(\boldsymbol{x})=\frac{1}{N} \sum_{i=1}^N\left(f_i(\boldsymbol{x}):=\mathbb{E}_{\xi_i\sim \mathcal{D}_i}\left[F_i\left(\boldsymbol{x} ; \xi_i\right)\right]\right).
	\label{eq 1}
\end{equation}
The function $f_i$ represents the loss function on client $i$. $\mathbb{E}_{\xi_i \sim \mathcal{D}_i}[\cdot]$ denotes the conditional expectation with respect to  the sample $\xi_i$. $\xi_i$ is drawn from distribution $\mathcal{D}_i$ in client $i$.  $N$ is the number of clients.

\section{Challenges of  Muon in FL}
\vspace{-2mm}
\subsection{The Muon Optimizer} 
\vspace{-2mm}
\paragraph{Motivation}
Most parameters in neural networks are inherently matrix-valued (e.g., in linear layers or the Q/K/V components of attention mechanisms). However, conventional optimization algorithms such as SGD and AdamW treat these parameters as vectors, effectively flattening them during updates and thereby neglecting their matrix structure. Muon is specifically designed to address this limitation by operating Matrix Orthogonalization directly on update matrix. 

\vspace{-2mm}
\paragraph{The Muon Optimizer}
Muon has recently been proposed as an optimization method for training neural network weights that can be represented as matrices. At iteration \( t \), given the current weight \( \mathbf{W}_{t-1} \), momentum \( \beta \), learning rate \( \eta_t \), and the objective \( F(\mathbf{W}) \), the update rules for the Muon optimizer are:
\vspace{-2mm}
\begin{equation}
	\begin{aligned}
		\mathbf{M}_t &= \beta \mathbf{M}_{t-1} + \nabla F(\mathbf{W}_{t-1}); \\
		\mathbf{O}_t &= \mathrm{Newton\text{-}Schulz}(\mathbf{M}_t); \\
		\mathbf{W}_t &= \mathbf{W}_{t-1} - \eta_t \mathbf{O}_t.
	\end{aligned}
	\label{eq:muon-update}
\end{equation}
Here, \( \mathbf{M}_t \) represents the momentum of the gradient at iteration \( t \), initialized as a zero matrix when \( t = 0 \). In Eq.(\ref{eq:muon-update}), a Newton–Schulz iteration is employed to approximate the solution of \( (\mathbf{M}_t \mathbf{M}_t^\top)^{-1/2} \mathbf{M}_t \). Let \( \mathbf{U} \boldsymbol{\Sigma} \mathbf{V}^\top = \mathbf{M}_t \) be the singular value decomposition (SVD) of \( \mathbf{M}_t \). Then, we have
$(\mathbf{M}_t \mathbf{M}_t^\top)^{-1/2} \mathbf{M}_t = \mathbf{U} \mathbf{V}^\top,$
which orthogonalizes \( \mathbf{M}_t \) (see  Figure \ref{fig:SVD}(a)). Intuitively, this orthogonalization ensures that the update matrices remain isomorphic, preventing the weights from learning solely along a few dominant directions. All matrix orthogonalization operations in this paper are computed using five Newton-Schulz iterations, resulting in about 5\% higher computation time compared to AdamW \citep{jordan6muon}.


\subsection{Challenges of  Muon in FL}
 Despite the widespread use of Muon in centralized deep learning, its adaptation to federated settings remains largely unexplored. In this subsection, we analyze two fundamental challenges that hinder its effectiveness in FL settings.
\vspace{-1mm}
{\colorlet{shadecolor}{LightBlue}
	\begin{snugshade}
		\begin{center}
			{\bf(Challenge 1)} \textit{\textbf{In non-IID FL, Muon’s client-specific preconditioner scales gradients from the client’s local data distribution, causing misalignment and cancellation in aggregation.}}
		\end{center}
		\vspace{-0.2cm}
\end{snugshade}}
\vspace{-2mm}
\textbf{Challenge Analysis:} The matrix orthogonalization in Muon can be viewed  as applying a client-specific linear preconditioner $P_i$ to each client's gradient (which can be approximated by  Newton-Schulz), transforming the update direction from $g_i$ to $P g_i$. In the case of   non-IID, the gradients $\{g_i\}$ are distributed across their respective dominant subspaces, and the $ P_i $ are independently estimated from the local data geometry of each client. This leads to direction mismatch and correlation/amplification: the global update is approximated as 
$
\sum_i \tilde{g}_i = \sum_i  P_i g_i.
$
When the $ \{P_i\} $ apply different "rotations/scalings" to the gradient subspaces across clients, the sign and magnitude of $ \langle \tilde{g}_i, \tilde{g}_j \rangle $ fluctuate significantly, making it prone to direction cancellation (weakening the norm and making step size ineffective) or phase misalignment (leading to oscillations as it crosses stable regions). These mechanisms together result in the phenomenon of \textbf{local–global inconsistency}: the convergence shown on the client side (local loss decreases rapidly) does not translate into global progress (global loss/accuracy stagnates or degrades).
\vspace{-1mm}
{\colorlet{shadecolor}{LightRed}
	\begin{snugshade}
		\begin{center}
			{\bf(Challenge 2)} \textit{\textbf{Moment reinitialization:} reinitializing the moment of Muon from scratch in every round hinders the convergence rate.}
		\end{center}
		\vspace{-0.2cm}
\end{snugshade}}
\vspace{-2mm}
\textbf{Challenge Analysis:} In FL, the Muon optimizer state is reinitialized to zero at the beginning of each round, i.e., $\boldsymbol{M}^{r,0}_i \gets \boldsymbol{0}$. This reset erases temporal memory across rounds, preventing the accumulation of momentum and thereby slowing convergence. Moreover, accumulating momentum from zero exacerbates client drift.

	\begin{algorithm}[H]
		\caption{\texttt{FedMuon} Algorithm}
		\begin{algorithmic}[1]
			\STATE {\textbf{Initial} model $\boldsymbol{x}^0$, $\beta=0.98$, the number of all clients $N$, each round selected clients $S$.}
			\FOR{$r = 1, \dots, R$}
			\FOR{each selected client $i \in \{1, \dots, S\}$ in parallel}
			\STATE $\boldsymbol{x}_{i}^{r,0} \gets \boldsymbol{x}^r$, \fcolorbox{LightRed}{LightRed}{$\boldsymbol{M}^{r,0}_{i} \gets \boldsymbol{\bar{M}}^{r}$;}
			\FOR{$k = 1, \dots, K$}
			\STATE $\boldsymbol{G}^{r,k}_i\gets\nabla f_i(\boldsymbol{x}_i^{r, k} ; \xi_i)$;
			 $\boldsymbol{M}^{r,k}_i = \beta \boldsymbol{M}^{r,k-1}_i + \boldsymbol{G}^{r,k}_i$;
			\STATE $\boldsymbol{U}^{r,k}_i, \boldsymbol{\Sigma}^{r,k}_i, \boldsymbol{V}^{r,k}_i = \mathrm{SVD}(\boldsymbol{M}^{r,k}_i)$;
			{\fcolorbox{LightBlue}{LightBlue}{$\boldsymbol{x}^{r,k+1}_i\! =\!\boldsymbol{x}^{r,k}_i \!\!-\! \eta [(1\!-\!\alpha)\boldsymbol{U}^{r,k}_i {\boldsymbol{V}^{r,k}_i}^{\top}\!+\!\alpha\boldsymbol{\Delta}_G^r ]$;}}
			\ENDFOR
			\STATE Communicate $(\boldsymbol{x}^{r, K}_i\!-\!\boldsymbol{x}^{r, 0}_i,  \boldsymbol{M}^{r,K}_i)$ to Server;
			\ENDFOR
			\STATE $\boldsymbol{\Delta}_G^{r+1}=-\frac{1}{SK\eta} \sum_{i=1}^S (\boldsymbol{x}^{r, K}_i-\boldsymbol{x}^{r, 0}_i)$; $\boldsymbol{x}^{r+1} =\boldsymbol{x}^{r} +\frac{1}{S} \sum_{i=1}^S (\boldsymbol{x}^{r, K}_i-\boldsymbol{x}^{r, 0}_i)$;
			\STATE \fcolorbox{LightRed}{LightRed}{$\boldsymbol{\bar{M}}^{r+1} =\frac{1}{S} \sum_{i=1}^S  \boldsymbol{M}^{r,K}_i$;}
			Communicate $(\boldsymbol{x}^{r+1}, \boldsymbol{\bar{M}}^{r+1},\boldsymbol{\Delta}_G^{r+1} ) $ to Clients.
			\ENDFOR
		\end{algorithmic}
		\label{algorithm_FedMuon}
	\end{algorithm}

\begin{figure}[tb]
	\includegraphics[width=0.95\textwidth]{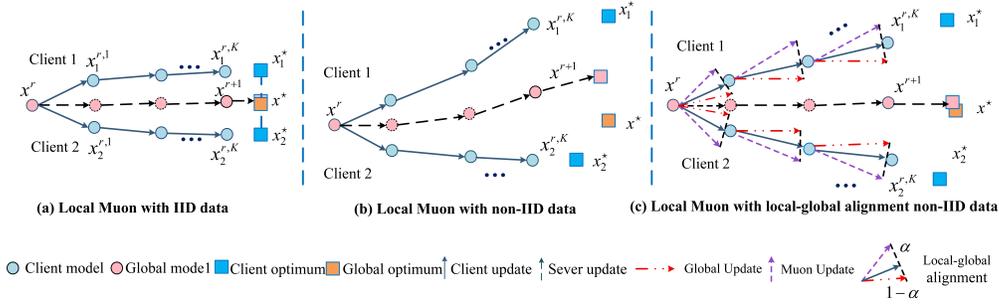}
	\vspace{-2mm}
	\caption{\small{An illustration of  \texttt{FedMuon}, which corrects client drift  through local-global alignment. }}
	\label{fig 4}
	
\end{figure}

\section{Our Algorithm: FedMuon}
\vspace{-2mm}
To robustly leverage matrix orthogonalization in FL, we propose \texttt{FedMuon}, with two core mechanisms for the non-IID regime.
\subsection{Mechanism I: Local--Global Alignment}

{\colorlet{shadecolor}{LightBlue}
	\begin{snugshade}
		\begin{center}
	{\bf (Q1)} 
\textit{\textbf{How to overcome \textbf{local–global inconsistency} in Local Muon?}}
		\end{center}
		\vspace{-0.2cm}
\end{snugshade}}
\vspace{-2mm}
To address  \textbf{Challenge 1}, we incorporate  local-global alignment  into the local update rule:
\vspace{-2mm}
\begin{equation}
\boldsymbol{x}^{r,k+1}_i =\boldsymbol{x}^{r,k}_i \!-\! \eta \big[(1-\alpha)\boldsymbol{U}^{r,k}_i {\boldsymbol{V}^{r,k}_i}^{\top}+\alpha\boldsymbol{\Delta}_G^r \big] ,
\end{equation}
where $\boldsymbol{\Delta}_G^r\!=\!-\frac{1}{SK\eta} \sum_{i=1}^S (\boldsymbol{x}_i^{r,K} \!-\! \boldsymbol{x}_i^{r,0})$ is the estimated global update. $\alpha$ is the trade-off coefficient between local and global updates.
 As shown in \textbf{Figure~\ref{fig 4}}, this alignment reduces the divergence of local models and improves global consistency. We also validate its effectiveness in the following experiments (see \textbf{Table~\ref{tab:ablation}} below).


\subsection{Mechanism II: Momentum Aggregation}
\vspace{-2mm}
{\colorlet{shadecolor}{LightRed}
	\begin{snugshade}
		\begin{center}
{\bf (Q2)} 
\textbf{\textit{How to initialize momentum of Muon in local?}}
		\end{center}
		\vspace{-0.2cm}
\end{snugshade}}
\vspace{-2mm}
To achieve better initialization of the momentum $\boldsymbol{M}$ in local, we aggregate local  momentum $\boldsymbol{M}^{r,K}_i$ and transmit the aggregated result $\boldsymbol{\bar{M}}$ back to the clients. This strategy partially mitigates the client drift caused by reinitializing momentum from zero, and better aligns local updates with the global update direction (see \textbf{Table~\ref{tab:ablation}} below).
\vspace{-1mm}
{\colorlet{shadecolor}{LightRed}
	\begin{snugshade}
		\begin{center}
			{\bf (Q3)} 
			\textbf{\textit{How to efficiently communicate momentum matrices?}}
		\end{center}
		\vspace{-0.2cm}
\end{snugshade}}
\vspace{-2mm}
\textbf{Momentum Compression via SVD.} 
Directly communicating the full momentum matrix $M$ would introduce prohibitive communication overhead. 
To reduce the cost, we compress $M$ using singular value decomposition (SVD):  
$M = U \Sigma V^\top,$
where $U$ and $V$ are orthogonal matrices and $\Sigma$ is the diagonal matrix of singular values. 
Instead of transmitting the full decomposition, we retain only the top-$k$ singular values (with $k$ set to 5\% of the matrix rank), yielding a low-rank approximation (see  Figure \ref{fig:SVD}):
$
M \approx U_k \Sigma_k V_k^\top.
$
This significantly reduces the communication cost $95\%$. We refer to this variant as \texttt{FedMuon\_SVD}. In the following experiments, we show that this approach achieves performance comparable to \texttt{FedMuon} (see Table \ref{tab:avg}).

\section{Theoretical Analysis}

\begin{table*}[tb]
	\centering
	\caption{Theoretical Comparison of \texttt{FedMuon} and Baseline Federated Optimization Methods} 
\vspace{-2mm}
	\label{table 1}
	\setlength{\tabcolsep}{0pt}	\begin{tabular}{llcccccc}
		\midrule[1pt]
		\centering
		& \textbf{Research work}   &  \textbf{Consider Matrix Structures } &\textbf{Convergence Rate}& \textbf{Assumption} \\
		\midrule[0.5pt]		
		&Local SGD & $\times$ &  $\mathcal{O}\left(\sqrt{\frac{L \Delta (\sigma_l^2+\sigma_g^2)}{S K R}}+\frac{L \Delta}{R}\right)$&Data Heterogeneity\\
		
		\rowcolor{LightRed}
		&\texttt{Local Muon}  & \checkmark &  $\mathcal{O}\left(\sqrt{\frac{L \Delta (\sigma_l^2+\sigma_g^2)}{S K R}}+\frac{L \Delta}{R}\right)$&Data Heterogeneity\\
		\rowcolor{LightBlue}
		&\texttt{FedMuon} (ours) & \checkmark &  $\mathcal{O}\left(\sqrt{\frac{L \Delta \sigma_l^2}{S K R}}+\frac{L \Delta}{R}\right)$&Without data heterogeneity\\
		\midrule[1.5pt]
	\end{tabular}
	\label{table com}
\end{table*}
\label{convergence_analysis}
In this part, we give the convergence theoretical analysis of our proposed \texttt{FedMuon} algorithm. Firstly we state some standard assumptions for the non-convex function $f$.
\begin{assumption}[Smoothness]
	\label{smoothness}
	 \textit{The non-convex $f_{i}$ is a $L$-smooth function for all $i\in[m]$, i.e., $\Vert\nabla f_{i}(\boldsymbol{x})-\nabla f_{i}(\boldsymbol{y})\Vert\leq L\Vert\boldsymbol{x}-\boldsymbol{y}\Vert$, for all $\boldsymbol{x},\boldsymbol{y}\in\mathbb{R}^{d}$.}
\end{assumption}
\begin{assumption}[Bounded Stochastic Gradient]
	\label{bounded_stochastic_gradient_I}
	\textit{$\boldsymbol{g}_{i}^{r}=\nabla f_{i}(\boldsymbol{x}_{i}^{r}, \xi_i^{r})$ computed by using a sampled mini-batch data $\xi_i^{r}$ in the local client $i$ is an unbiased estimator of $\nabla f_{i}$ with bounded variance, i.e., $\mathbb{E}_{\xi_i^{r}}[\boldsymbol{g}_{i}^{r}]=\nabla f_{i}(\boldsymbol{x}_{i}^{r})$ and     $\mathbb{E}_{\xi_i^{r}}\Vert g_{i}^{r} - \nabla f_{i}(\boldsymbol{x}_{i}^{r})\Vert^{2} \leq \sigma_{l}^{2}$, for all $\boldsymbol{x}_{i}^{r}\in\mathbb{R}^{d}$.}
\end{assumption}

\begin{assumption}[Bounded Heterogeneity]
	\label{bounded_heterogeneity}
	\textit{The dissimilarity between local clients is bounded on the gradients, i.e., $\Vert\nabla f_{i}(\boldsymbol{x})-\nabla f(\boldsymbol{x})\Vert^{2}\leq\sigma_{g}^{2}$, for all $\boldsymbol{x}\in\mathbb{R}^{d}$.}
\end{assumption}
These assumptions are standard in FL optimization literature \citep{fan2024locally,sun2023efficient}.
\begin{theorem}[Convergence for non-convex functions]\label{theorem_convergence_rate1}
	Under Assumptions \ref{smoothness}, \ref{bounded_stochastic_gradient_I}, if we take $g^0=0$,$\beta_1=0,\lambda=0$
	then \texttt{FedMuon} converges as follows
	\vspace{-2mm}
	\begin{equation}
		\frac{1}{R} \sum_{r=0}^{R-1} \mathbb{E}\left[\left\|\nabla f\left(\boldsymbol{x}^{r}\right)\right\|^2\right] \lesssim \mathcal{O}\left(\sqrt{\frac{L \Delta \sigma_l^2}{S K R }}+\frac{L \Delta}{R}\right) .
	\end{equation}
	Here $G_0:=\frac{1}{N} \sum_{i=1}^N\left\|\nabla f_i\left(\boldsymbol{x}^0\right)\right\|^2$,$\Delta=f\left(\boldsymbol{x}^0\right)-f^{\star} $, $S$ is the number of participating clients per round, $K$ is the number of local iterations, and $R$ is the total number of communication rounds. 
\end{theorem}	

The detailed proof is provided in the \textbf{Appendix}. 
As summarized in \textbf{Table~\ref{table 1}}, the convergence rate of \texttt{FedMuon} is 
$
\mathcal{O}\!\left(\sqrt{\tfrac{L \Delta (\sigma_l^2)}{S K R}} + \tfrac{L \Delta}{R}\right),
$
which is faster than than that of  both Local Muon and Local SGD, $
\mathcal{O}\!\left(\sqrt{\tfrac{L \Delta (\sigma_l^2+\sigma_g^2)}{S K R}} + \tfrac{L \Delta}{R}\right)
$. 
Notably, our result does not rely on \textbf{Assumption 3}. 
This improvement stems from the suppression of local drift achieved by the proposed local--global alignment mechanism. 
The effectiveness of this design is further validated in the \textbf{ablation study} (Table~\ref{tab:ablation}).

\begin{figure*}[tb]
	\centering
	\begin{minipage}[t]{0.24\textwidth}
		\centering
		\subcaptionbox{ResNet18, Dir-0.1}{\includegraphics[width=\textwidth]{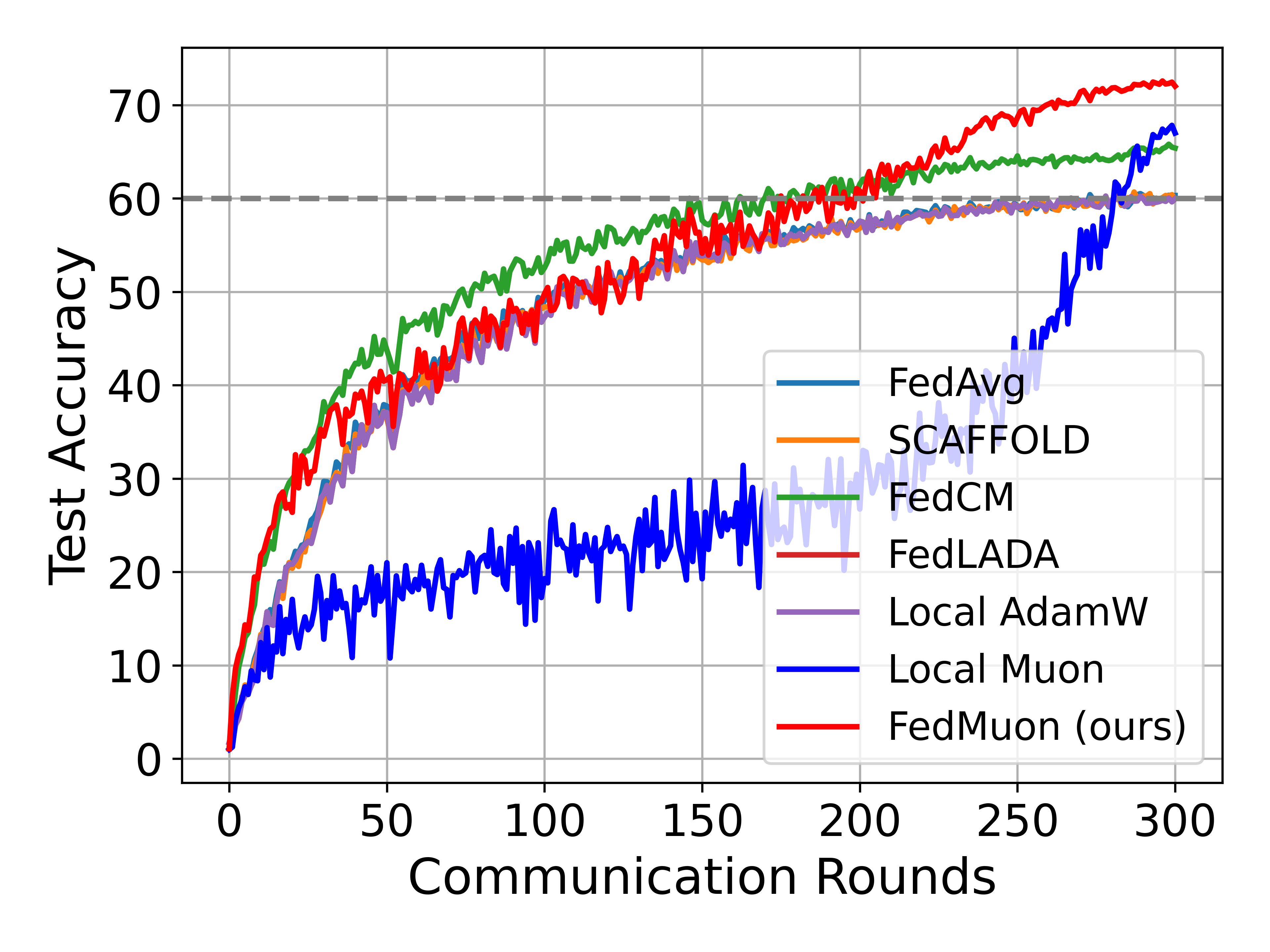}}
	\end{minipage}
	\begin{minipage}[t]{0.24\textwidth}
		\subcaptionbox{ResNet18, Dir-0.1}{\includegraphics[width=\textwidth]{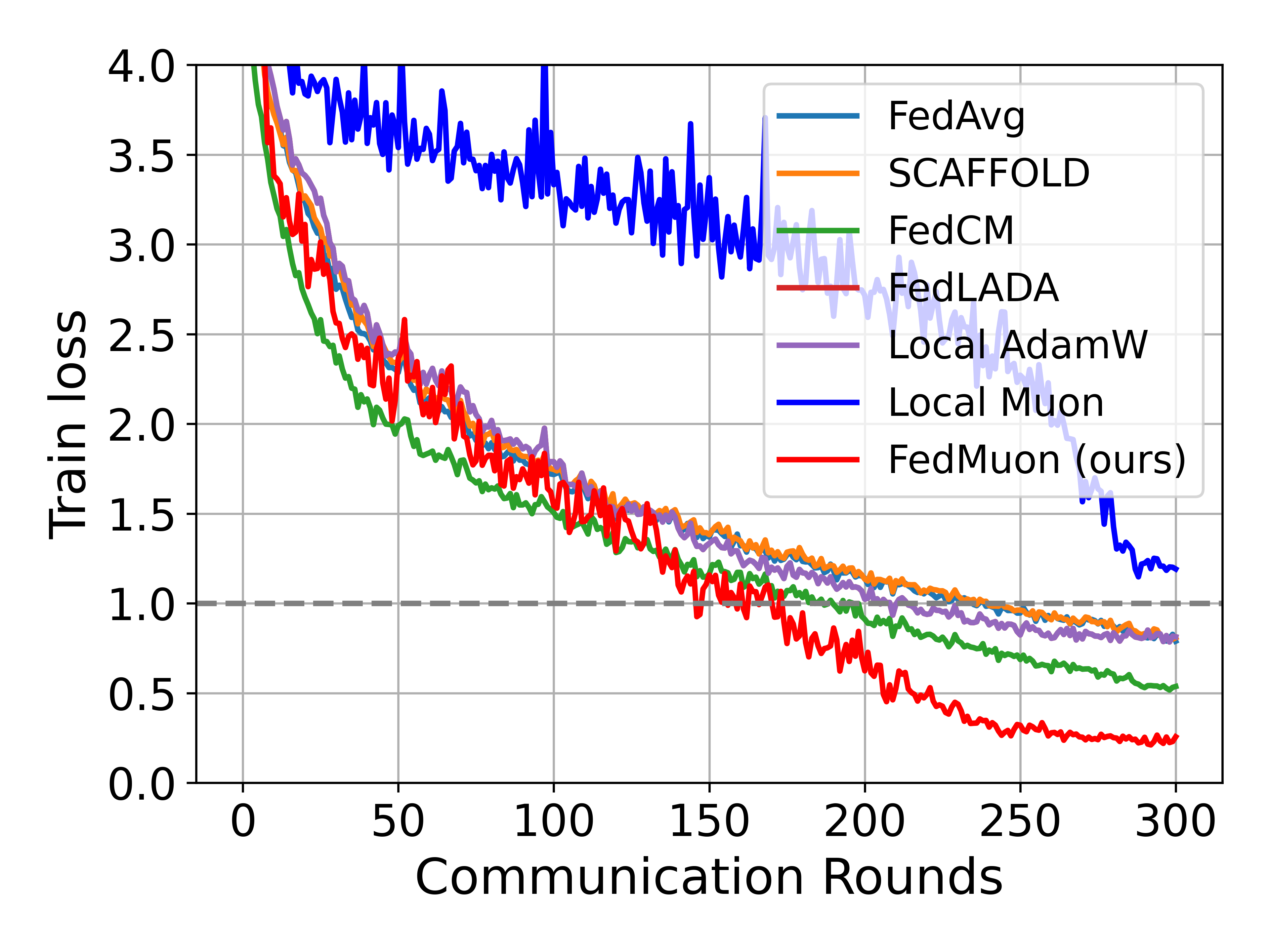}}
	\end{minipage}
	\begin{minipage}[t]{0.24\textwidth}
		\subcaptionbox{ViT-Tiny, Dir-0.1}{\includegraphics[width=\textwidth]{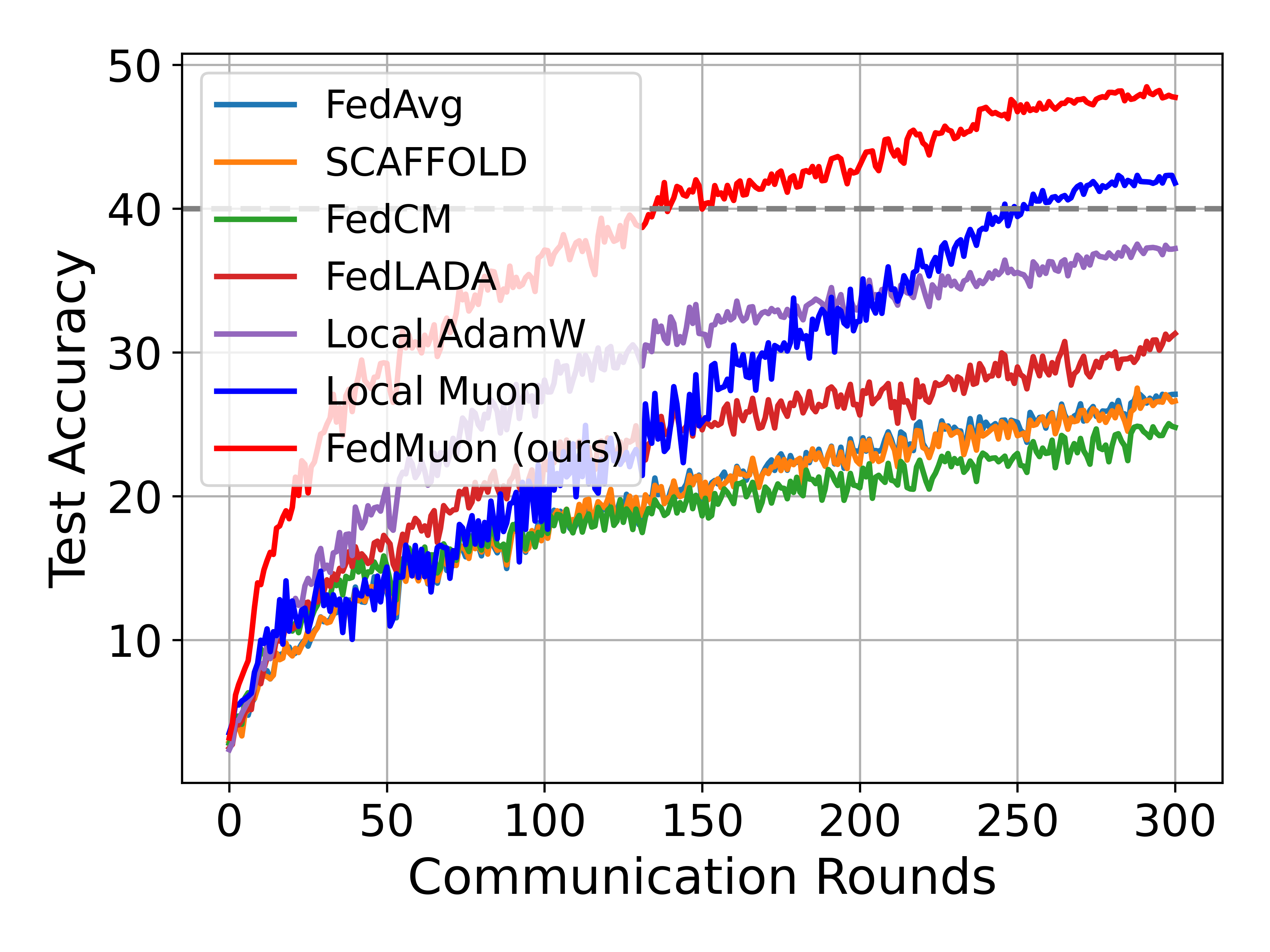}}
	\end{minipage}
	\begin{minipage}[t]{0.24\textwidth}
		\subcaptionbox{ViT-Tiny, Dir-0.1}{\includegraphics[width=\textwidth]{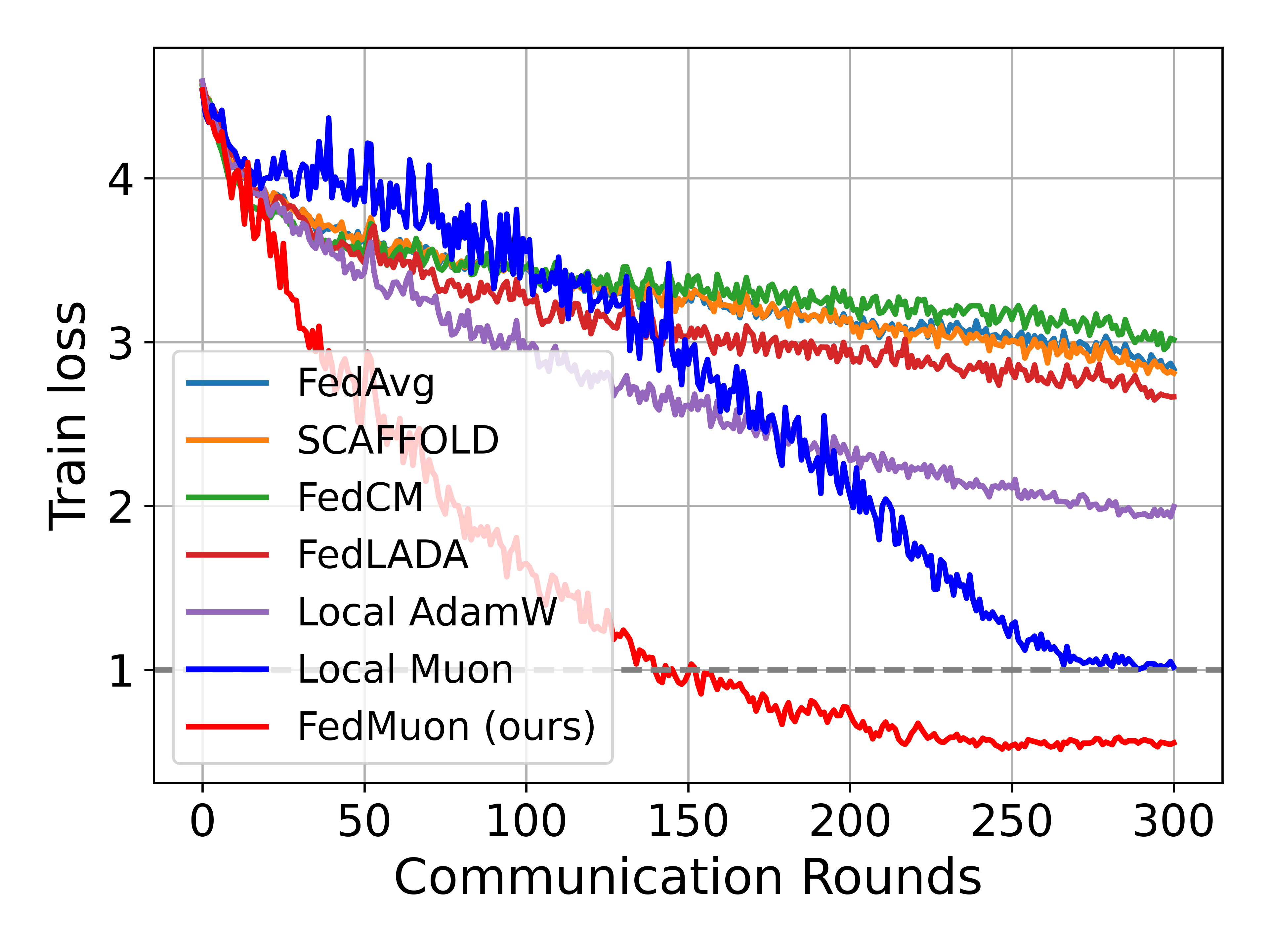}}
	\end{minipage}
		\caption{\small Training loss and Test acc curves on CIFAR-100 using ResNet-18 and ViT-Tiny in Dir-0.1.}
		\label{fig:resnet18}
	\end{figure*}

	\begin{table*}[tb]
		\centering
		\caption{\small Test accuracy, training loss of each method on CIFAR-100 using \textbf{ResNet-18} and \textbf{ViT-Tiny} over 300 communication rounds under Dir-0.6 and Dir-0.1 (100 clients, 10\% participation, batch size 50, $K=50$).}
		\label{tab:combined_cifar100}
		\setlength{\tabcolsep}{2pt}
		\begin{tabular}{lccccccccc}
			\toprule
			\multirow{2}{*}{\textbf{Method}} 
			& \multicolumn{2}{c}{\textbf{ResNet-18 (Dir-0.6)}} 
			& \multicolumn{2}{c}{\textbf{ResNet-18 (Dir-0.1)}}
			& \multicolumn{2}{c}{\textbf{ViT-Tiny (Dir-0.6)}} 
			& \multicolumn{2}{c}{\textbf{ViT-Tiny (Dir-0.1)}}  \\
			\cmidrule(lr){2-3} \cmidrule(lr){4-5} 
			\cmidrule(lr){6-7} \cmidrule(lr){8-9}
			& Test Acc & Loss & Test Acc & Loss  
			& Test Acc & Loss   & Test Acc & Loss  \\
			\midrule
			FedAvg         & $64.08_{\pm0.18}$ & 0.376 & $60.25_{\pm0.20}$ & 0.767 & $32.36_{\pm0.08}$ & 2.350 & $27.14_{\pm0.12}$ & 2.867 \\
			SCAFFOLD       & $65.01_{\pm0.19}$ & 0.365 & $59.37_{\pm0.16}$ &0.814& $32.17_{\pm0.12}$ & 2.295& $27.31_{\pm0.11}$ & 2.855 \\
			FedCM          & $70.42_{\pm0.11}$ & 0.282 & $66.73_{\pm0.14}$ & 0.639 & $26.33_{\pm0.12}$ & 2.681 & $23.18_{\pm0.12}$ & 3.038 \\
			FedLADA        & $65.07_{\pm0.17}$ & 0.671 & $57.78_{\pm0.18}$ & 0.498 & $38.33_{\pm0.12}$ & 2.121& $31.50_{\pm0.12}$ & 2.678 \\
			Local AdamW    & $62.84_{\pm0.08}$ &0.363& $58.97_{\pm0.10}$ & 0.794 & $40.47_{\pm0.09}$ & 1.026& $37.86_{\pm0.11}$ & 1.954  \\
			\rowcolor{LightBlue}
			\texttt{Local Muon}    & $71.66_{\pm0.15}$ & 0.395 &$66.71_{\pm0.15}$ & 1.504& $46.69_{\pm0.12}$ & 0.201 & $40.53_{\pm0.12}$ & 1.432  \\
			\rowcolor{LightRed}
			\texttt{FedMuon} 
			&$\mathbf{74.12_{\pm0.15}}$ & 0.001 
			& $\mathbf{73.05_{\pm0.15}}$ &0.246 
			& $\mathbf{50.22_{\pm0.12}}$ & 0.162 
			& $\mathbf{48.18_{\pm0.12}}$  & 0.556  \\
			\bottomrule
		\end{tabular}
	\end{table*}

	\section{Experiments}
	\vspace{-2mm}
	\textbf{Datasets.} We evaluate \texttt{FedMuon} on both vision and language tasks. (\textit{i}) For image classification, we use CIFAR-100~\citep{krizhevsky2009learning}, and Tiny ImageNet~\citep{le2015tiny}. (\textit{ii}) For NLP tasks, we adopt benchmark datasets from the GLUE benchmark, including SST-2~\citep{socher2013recursive}, QQP~\citep{socher2013recursive}. To simulate data heterogeneity across clients, we follow the Dirichlet partitioning scheme~\citep{hsu2019measuring}, where a Dir-0.6 corresponds to a low heterogeneity and Dir-0.1 implies high heterogeneity.\\
	\textbf{Model Architectures.} We explore a variety of model types: (\textit{i}) ResNet-18~\citep{he2016deep} as a representative convolutional neural network (CNN), (\textit{ii}) Swin Transformer~\citep{liu2021swin} and ViT-Tiny~\citep{dosovitskiy2020image} for Vision Transformers, and (\textit{iii}) RoBERTa-Base~\citep{liu2019roberta} for large-scale language model.\\
	\textbf{Baselines.} We compare our method against  state-of-the-art FL algorithms: \texttt{FedAvg} (\texttt{Local SGD})~\citep{mcmahan2017communication}, \texttt{SCAFFOLD}~\citep{karimireddy2020scaffold}, \texttt{FedCM}~\citep{xu2021fedcm}, \texttt{FedLADA}~\citep{sun2023efficient}, \texttt{Local AdamW} and  \texttt{Local Muon}.\\
	\textbf{Hyperparameter Settings.} For \texttt{FedAvg}, \texttt{SCAFFOLD}, \texttt{FedCM}, the $lr$ is selected from 
	$\{10^{-2},\ 3 \times 10^{-2},\ 5 \times 10^{-2},\ 10^{-1},\ 3 \times 10^{-1}\}$, 
	with a weight decay of $0.001$.
	For \texttt{FedLADA}, \texttt{Local AdamW}, the $lr$ is selected from 
	$\{10^{-4},\ 3 \times 10^{-4},\ 5 \times 10^{-4},\ 8 \times 10^{-4},\ 10^{-3}\}$, 
	with  weight decay  $0.01$ or $0.001$, $\beta_1 = 0.9$, $\beta_2 = 0.999$. We apply cosine learning rate decay, and set \texttt{FedMuon} to \textbf{$\boldsymbol{\alpha}\!=\!0.5$}, weight decay $0.01$. We set the learning rate of \texttt{FedMuon} and \texttt{Local Muon} to be $ 3 \times 10^{-2},\ 2 \times 10^{-2},\ 3 \times 10^{-3} $. Additional hyperparameter configurations are detailed in the \textbf{Appendix}. We release all code, configuration files to ensure full reproducibility. All results are averaged over 5 runs with std reported with seeds {42, 43, 44, 45, 46}.
	

	\subsection{Results on Convolutional Neural Networks and Transformer}
	
\paragraph{Training on CIFAR-100 with ResNet-18.}
\textbf{Table~\ref{tab:combined_cifar100}} and \textbf{Figure~\ref{fig:resnet18}} present the test accuracy and training loss on CIFAR-100 using ResNet-18. \texttt{FedMuon} achieves the best performance under both Dir-0.6 and Dir-0.1 settings, reaching a top accuracy of \textbf{74.12\%} and \textbf{73.05\%}, respectively. It also attains the lowest training loss (\textbf{0.001} and \textbf{0.246}), demonstrating faster and more stable convergence. Compared to other adaptive baselines such as  \texttt{Local AdamW}, \texttt{FedMuon} shows superior generalization under data heterogeneity, confirming its effectiveness in CNNs.
	
\textbf{Training on CIFAR-100 with ViT-Tiny.}
\textbf{Table~\ref{tab:combined_cifar100}} and \textbf{Figure~\ref{fig:resnet18}} show
\texttt{FedMuon} achieves the best performance across both heterogeneity levels, with  \textbf{50.22\%} (Dir-0.6) and \textbf{48.18\%} (Dir-0.1), and the lowest training loss (\textbf{0.162} and \textbf{0.556}), confirming its efficient convergence. Compared to \texttt{Local AdamW}, it provides consistent improvements in both accuracy and stability. Moreover, other adaptive baselines such as  \texttt{FedLADA} perform significantly worse under high heterogeneity, highlighting the effectiveness of global update correction and decoupled weight decay.
These results validate that \texttt{FedMuon} is particularly effective for federated vision Transformers under non-i.i.d. conditions.
The small dataset CIFAR100 is difficult to support the performance of ViT, resulting in lower accuracy. Therefore, we continued to test on the pretrained model.\\
	
	\begin{wraptable}{r}{0.54\textwidth}
		\centering
		\setlength{\tabcolsep}{1pt}          
		\renewcommand{\arraystretch}{1.15}
		\captionof{table}{\small Comparison of test accuracy and training loss for \textbf{Swin Transformer} under Dir-0.1 with 100  rounds (100 clients, 5\% participation, batch size 16, $K=50$).}
		\vspace{-2mm}
		\label{tab:swin_results}
		\begin{tabular}{lcccc}
			\toprule
			\multirow{2}{*}{\textbf{Method}} 
			& \multicolumn{2}{c}{\textbf{CIFAR-100}} 
			& \multicolumn{2}{c}{\textbf{Tiny ImageNet}} \\
			\cmidrule(lr){2-3} \cmidrule(lr){4-5}
			& Test Acc &Loss & Test Acc & Loss \\
			\midrule
			FedAvg         & $80.02_{\pm0.28}$  & 0.588 & $80.38_{\pm0.22}$  & 0.826 \\
			SCAFFOLD       & $81.30_{\pm0.18}$  & 0.514 & $82.41_{\pm0.18}$  & 0.650 \\
			FedCM          & $82.38_{\pm0.11}$  & 0.565 & $83.18_{\pm0.14}$  & 0.522\\
			FedLADA        & $74.64_{\pm0.15}$  & 0.598 &$70.94_{\pm0.19}$  & 0.944 \\
			Local AdamW     & $83.35_{\pm0.16}$  &0.381 & $80.26_{\pm0.12}$  & 0.686 \\
			\rowcolor{LightBlue}
			\texttt{Local Muon}     & $79.73_{\pm0.18}$ & 0.396 & $80.24_{\pm0.10}$            & 0.734 \\
			\rowcolor{LightRed}
			\texttt{FedMuon} &$\mathbf{84.88_{\pm0.17}}$  &$\mathbf{ 0.123}$ & $\mathbf{84.95_{\pm0.12}}$ &$\mathbf{ 0.394}$ \\
			\bottomrule
		\end{tabular}
	\end{wraptable}
	
	\textbf{Fine-tuning Results on Swin Transformer.}  
	\textbf{Table~\ref{tab:swin_results}} reports results on Swin Transformer under Dir-0.1. \texttt{FedMuon} achieves the highest test accuracy on both CIFAR-100 (\textbf{84.88\%}) and Tiny ImageNet (\textbf{84.95\%}), while also attaining the lowest training loss, reflecting faster convergence. FedMuon consistently outperforms baselines (including Local AdamW and Local Muon), demonstrating its effectiveness in fine-tuning Vision Transformer models under non-IID data.\\
	\textbf{Fine-tuning Results on LLMs.} 
	\textbf{Table~\ref{tab:roberta_base_glue}} summarizes results on the GLUE benchmark using RoBERTa-Base with LoRA, 4 clients, 100\% participation, batch size 16, $K=50$, rank$=$16. \texttt{FedMuon} achieves the highest  accuracy of GLUE outperforming strong baselines such as \texttt{FedAvg} and \texttt{Local Muon}. It is particularly strong on challenging tasks like \textbf{RTE} and \textbf{QQP}, exceeding the next best methods by \textbf{+5.64\%} and \textbf{+0.62\%}, respectively.

\begin{table*}[tb]
	\centering
	\setlength{\tabcolsep}{1pt}
	\caption{\small Test accuracy using RoBERTa-Base with LoRA across GLUE tasks over 100 communication rounds. }
	\vspace{-2mm}
	\label{tab:roberta_base_glue}
	\small
	\begin{tabular}{lccccccc}
		\toprule
		\textbf{Method (Dir-0.8)} 
		& \textbf{CoLA} & \textbf{RTE} & \textbf{SST-2} & \textbf{QQP}
		& \textbf{MRPC} & \textbf{QNLI} & \textbf{MNLI} \\
		\midrule
		FedAvg           
		& $57.32_{\pm0.22}$ & $62.71_{\pm0.35}$ & $93.32_{\pm0.08}$ 
		& $84.13_{\pm0.15}$ & $87.02_{\pm0.19}$ & $90.19_{\pm0.12}$ 
		& $84.18_{\pm0.21}$  \\
		
		SCAFFOLD         
		& $58.14_{\pm0.25}$ & $63.62_{\pm0.28}$ & $93.54_{\pm0.09}$
		& $84.62_{\pm0.17}$ & $87.56_{\pm0.22}$ & $90.26_{\pm0.11}$
		& $84.26_{\pm0.20}$ \\
		
		FedCM            
		& $58.14_{\pm0.27}$ & $66.14_{\pm0.31}$ & $93.61_{\pm0.07}$
		& $84.56_{\pm0.18}$ & $87.11_{\pm0.16}$ & $90.08_{\pm0.13}$ 
		& $84.32_{\pm0.23}$  \\
		
		FedLADA          
		& $59.10_{\pm0.21}$ & $74.14_{\pm0.29}$ & $93.66_{\pm0.10}$
		& $84.86_{\pm0.16}$ & $87.42_{\pm0.18}$ & $90.18_{\pm0.14}$ 
		& $84.46_{\pm0.19}$  \\
		
		Local AdamW      
		& $59.33_{\pm0.23}$ & $74.04_{\pm0.27}$ & $93.55_{\pm0.11}$ 
		& $84.68_{\pm0.15}$ & $87.16_{\pm0.20}$ & $90.11_{\pm0.12}$ 
		& $84.54_{\pm0.18}$  \\
		
		\rowcolor{LightBlue}
		\texttt{Local Muon}      
		& $60.16_{\pm0.20}$ & $71.48_{\pm0.34}$ & $93.34_{\pm0.09}$
		& $85.11_{\pm0.13}$ & $87.45_{\pm0.21}$ & $90.97_{\pm0.15}$ 
		& $84.59_{\pm0.17}$ \\
		
		\rowcolor{LightRed}
		\texttt{FedMuon (ours)} 
		& $\mathbf{63.04_{\pm0.19}}$ & $\mathbf{77.12_{\pm0.30}}$ & $\mathbf{94.12_{\pm0.08}}$
		& $\mathbf{85.73_{\pm0.14}}$ & $\mathbf{88.23_{\pm0.17}}$ & $\mathbf{91.43_{\pm0.10}}$  
		& $\mathbf{85.05_{\pm0.16}}$   \\
		\bottomrule
	\end{tabular}
\end{table*}

\subsection{Ablation Study}
\vspace{-2mm}
\textbf{Impact of $\Delta_G$ and $\bar{m}$.}	As shown in Table~\ref{tab:ablation} left, we conduct an ablation study of \texttt{FedMuon}. 
	\texttt{FedMuon} incorporates  momentum averaging $\bar{m}$ and global update differences $\Delta_G$. 
	The results clearly indicate that \texttt{Local Muon} consistently outperforms both SGD and AdamW, 
	demonstrating its superior ability to handle non-IID FL. Moreover, our strategy consistently improves the performance of other optimizers as well.

\textbf{Impact of $\Delta_G$ and $\bar{m}$ on other optimizers.}
As shown in Table~\ref{tab:ablation} right, we compare different local optimizers with $\Delta_G$ and $\bar{m}$. 
The results demonstrate that \texttt{Local Muon} consistently achieves the best performance, significantly outperforming SGD and AdamW, thereby highlighting its effectiveness in mitigating data heterogeneity.
	
		\begin{table*}[tb]
		\centering
		\caption{\small Ablation study of \texttt{FedMuon} on CIFAR-100 (Dir-0.1, 300 rounds). 
			Left: effect of removing components. Right: effect of different local optimizers.}
		
		\label{tab:ablation}
		\setlength{\tabcolsep}{4pt}
		\begin{minipage}{0.47\textwidth}
			\centering
			\begin{tabular}{lcc}
				\toprule
				\textbf{Variant} & \textbf{ResNet-18} & \textbf{ViT-Tiny} \\
				\midrule
				A1: w/o $\boldsymbol{\bar{m}}$        & $69.12_{\pm0.18}$ & $43.67_{\pm0.19}$ \\
				A2: w/o $\boldsymbol{\Delta}_G$       & $68.05_{\pm0.10}$ & $44.56_{\pm0.16}$ \\
				\rowcolor{LightRed}
				A3: \texttt{FedMuon}                  & $\mathbf{73.05_{\pm0.15}}$ &$\mathbf{48.18_{\pm0.12}}$\\
				\bottomrule
			\end{tabular}
		\end{minipage}
		\begin{minipage}{0.52\textwidth}
			\centering
			\setlength{\tabcolsep}{2pt}
			\begin{tabular}{lcc}
				\toprule
				\textbf{Variant} & \textbf{ResNet-18} & \textbf{ViT-Tiny} \\
				\midrule
				Local SGD + $\boldsymbol{\bar{m}}$+$\boldsymbol{\Delta}_G$   & $66.28_{\pm0.17}$ & $32.56_{\pm0.11}$\\
				Local AdamW + $\boldsymbol{\bar{m}}$+$\boldsymbol{\Delta}_G$ & $64.25_{\pm0.12}$ & $41.26_{\pm0.17}$ \\
				\rowcolor{LightRed}
				Local Muon + $\boldsymbol{\bar{m}}$+$\boldsymbol{\Delta}_G$  & $\mathbf{73.05_{\pm0.15}}$ &$\mathbf{48.18_{\pm0.12}}$ \\
				\bottomrule
			\end{tabular}
		\end{minipage}
	\end{table*}

	\begin{table}[tb]
		\centering
		\caption{\small Impact of $\alpha$ and $\beta$ on \texttt{FedMuon} using ViT-Tiny and ResNet-18  on CIFAR-100 (Dir-0.1).}
		\vspace{-2mm}
		\label{tab:alpha_beta_ablation}
		\setlength{\tabcolsep}{5pt}
		\begin{tabular}{l|ccccc|ccccc}
			\toprule
			\rowcolor{LightRed}
			& \multicolumn{5}{c|}{\textbf{$\alpha$}} & \multicolumn{5}{c}{\textbf{$\beta$}} \\
			\cmidrule(lr){2-6} \cmidrule(lr){7-11}
			Model & 0.00 & 0.25 & \textbf{0.50} & 0.75 & 0.90 & 0.80 & 0.90 & 0.95 & \textbf{0.98} & 0.99 \\
			\midrule
			\textbf{ResNet-18} & $68.05$ & $69.89$ & $\mathbf{73.01}$& $72.12$ & $67.56$ & $68.22$ & $70.56$ & $71.23$ & $\mathbf{73.01}$ & $72.66$ \\
			\textbf{ViT-Tiny}       & $44.56$     & $46.28$    & $\mathbf{48.18}$    & $47.59$     & $46.23$     & $44.86$     & $45.23$     & $46.59$     & $\mathbf{48.18}$      & $47.56$ \\
			\bottomrule
		\end{tabular}
	\end{table}

\textbf{Impact of $\alpha$.}  
\textbf{Table~\ref{tab:alpha_beta_ablation}} evaluates the effect of the local-global alignment parameter $\alpha$ in \texttt{FedMuon}. As shown by  \textbf{Theorem~1}, incorporating global update direction helps suppress client drift and accelerates convergence. We observe that $\alpha=0.5$ yields the best performance, striking a balance between local adaptivity and global consistency, in line with our theoretical insight.

\textbf{Impact of $\beta$.}  
\textbf{Table~\ref{tab:alpha_beta_ablation}} verifies the effectiveness of \emph{local momentum accumulation}. When the momentum parameter \(\beta\) is too small, the aggregated global momentum is quickly diluted. Conversely, an overly large \(\beta\) slows local gradient accumulation and delays responsiveness to new data. These results suggest that \(\beta\) should balance global momentum preservation with timely adaptation to client updates. We observe that $\beta=0.98$ yields the best performance.

	\begin{table}[H]
		\centering
		\caption{\small Ablation study of momentum aggregation strategies of \texttt{FedMuon} on CIFAR-100 under Dir-0.1.}
		\vspace{-2mm}
		\label{tab:avg}
		\setlength{\tabcolsep}{4pt}
		\begin{tabular}{lccc}
			\toprule
			\textbf{Aggregation Strategy} & \textbf{ResNet-18} & \textbf{ViT-Tiny}   & \textbf{Comm Cost($\downarrow$)} \\
			\midrule
			NoAgg      & 69.12$_{\pm0.15}$ &43.67$_{\pm0.15}$ & 1$\times$  \\
			Agg-$\boldsymbol{m}$                   & 73.05$_{\pm0.15}$ &  48.18$_{\pm0.13}$ & 2$\times$  \\
			Agg-$\boldsymbol{m}$-SVD     & $72.56_{\pm0.12}$ & $47.66_{\pm0.12}$ & 1.05$\times$  \\          
			
			\bottomrule
		\end{tabular}
	\end{table}

\textbf{Impact of Momentum Aggregation Strategy.} \textbf{Table \ref{tab:avg}} shows Momentum Aggregation Strategy, \texttt{Agg-$\boldsymbol{m}$-SVD} (\texttt{FedMuon\_SVD}), achieves the best balance between accuracy and communication cost. While \texttt{Agg-$\boldsymbol{m}$} improves performance, it introduces excessive communication (2$\times$). In contrast, \texttt{Agg-$\boldsymbol{m}$-SVD}  attains similar benefits with only 1.05$\times$ communication cost.

\section{Conclusion}
\vspace{-2mm}		
In this work, we proposed \texttt{FedMuon}, a structure-aware federated optimizer for training large-scale Transformer and vision models. \texttt{FedMuon} addresses core challenges of non-IID. Federated learning—client drift, unstable optimizer states, and inefficient communication—by coupling \emph{matrix-orthogonalized} local updates with \emph{local-global alignment} and \emph{cross-round momentum aggregation}, complemented by low-rank state sharing. We provided non-convex convergence analysis clarifying how alignment and orthogonalization jointly control the bias introduced by multi-step local training, and we documented strong empirical gains across vision and language tasks, particularly on Transformer architectures. These results highlight that treating optimizer updates as matrices (rather than flat vectors) offers a principled route to reliable and efficient FL. We believe \texttt{FedMuon} opens a pathway for adapting modern, structure-aware optimizers to federated settings and inspires future extensions to related methods such as LAMB \citep{chen2023symbolic} or Lion \citep{chen2023symbolic}. Beyond federated learning, the principles of \texttt{FedMuon} can be directly applied to large-scale distributed training and parameter-efficient fine-tuning of foundation models, where communication efficiency and stable optimization are equally critical.

\newpage

\section{Reproducibility Statement}
We make every effort to ensure reproducibility. The paper specifies training steps, model configurations (e.g., ResNet/ViT for vision and RoBERTa-style encoders for NLP), non-IID partition protocols, client sampling, and hardware details. Unless noted otherwise, each configuration is repeated with five independent seeds $\{42,43,44,45,46\}$; we report mean~$\pm$~standard deviation and provide per-run logs/curves. Implementation details for \textsc{FedMuon} (orthogonalized updates, global--local alignment, cross-round momentum aggregation, and low-rank SVD compression) are described in algorithmic form with all tunables exposed. An anonymous repository includes source code, configuration files, data-partition scripts, and instructions to exactly reproduce the main tables and figures.


\bibliography{main}
\bibliographystyle{iclr2026_conference}

\newpage

\section{LLM Usage}
Large Language Models (LLMs) were used solely for language editing (grammar, phrasing, and clarity) of the manuscript text. LLMs were \emph{not} involved in research ideation, methodological design, theoretical analysis, dataset preparation, implementation, or result selection. The authors are fully responsible for the scientific content and verify that any LLM-assisted passages comply with ethical guidelines and do not constitute plagiarism or scientific misconduct.

\appendix


\onecolumn
\appendix

\section{Appendix A: Proof of Theorem 1 and Convergence Analysis}
\label{convergence_analysis}

\begin{algorithm}[tb]
	\caption{\texttt{FedMuon} Algorithm}
	\begin{algorithmic}[1]
		\STATE {\textbf{Initial} model $\boldsymbol{x}^0$, $\beta_1=0.98$, time step $t \leftarrow 0$, the number of all clients $N$, each round selected clients $S$, weight decay $\lambda$}.
		\FOR{$r = 1, \dots, R$}
		\FOR{each selected client $i \in \{1, \dots, S\}$ in parallel}
		\STATE $\boldsymbol{x}_{i}^{r,0} \gets \boldsymbol{x}^r$, $\boldsymbol{M}^{r,0}_{i} \gets \boldsymbol{\bar{M}}^{r}$;
		\FOR{$k = 1, \dots, K$}
		\STATE $\boldsymbol{G}^{r,k}_i\gets\nabla f_i(\boldsymbol{x}_i^{r, k} ; \xi_i)$;
		\STATE $\boldsymbol{M}^{r,k}_i = \beta \boldsymbol{M}^{r,k}_i + \boldsymbol{G}^{r,k}_i$;
		\STATE $\boldsymbol{U}^{r,k}_i, \boldsymbol{\Sigma}^{r,k}_i, \boldsymbol{V}^{r,k}_i = \mathrm{SVD}(\boldsymbol{M}^{r,k}_i)$;
		\STATE$\boldsymbol{x}^{r,k+1}_i =\boldsymbol{x}^{r,k}_i - \eta_t \big[(1-\alpha)\boldsymbol{U}^{r,k}_i {\boldsymbol{V}^{r,k}_i}^{\top} + \lambda \boldsymbol{x}^{r,k}_i+\alpha\boldsymbol{\Delta}_G^r \big]$;
		
		\ENDFOR
		\STATE Communicate $( \boldsymbol{x}^{r, K}_i-\boldsymbol{x}^{r, 0}_i,  \boldsymbol{M}^{r,k}_i )$ to Server;
		\ENDFOR
		\STATE $\boldsymbol{\Delta}_G^r=\frac{-1}{SK\eta} \sum_{i=1}^S (\boldsymbol{x}^{r, K}_i-\boldsymbol{x}^{r, 0}_i)$;
		\STATE $\boldsymbol{x}^{r+1} =\boldsymbol{x}^{r} +\frac{1}{S} \sum_{i=1}^S (\boldsymbol{x}^{r, K}_i-\boldsymbol{x}^{r, 0}_i)$;
		\STATE $\boldsymbol{\bar{M}}^{r+1} =\frac{1}{S} \sum_{i=1}^S  \boldsymbol{M}^{r,k}_i$;
		\STATE Communicate $(\boldsymbol{x}^{r+1}, \boldsymbol{\bar{M}}^{r+1},\boldsymbol{\Delta}_G^r ) $ to Clients.
		\ENDFOR
	\end{algorithmic}
	\label{algorithm_fedadamw}
\end{algorithm}

\begin{algorithm}[tb]
	\caption{\texttt{FedMuon-QR} Algorithm}
	\begin{algorithmic}[1]
		\STATE {\textbf{Initial} model $\boldsymbol{x}^0$, $\beta_1=0.98$, time step $t \leftarrow 0$, the number of all clients $N$, each round selected clients $S$, weight decay $\lambda$}.
		\FOR{$r = 1, \dots, R$}
		\FOR{each selected client $i \in \{1, \dots, S\}$ in parallel}
		\STATE $\boldsymbol{x}_{i}^{r,0} \gets \boldsymbol{x}^r$, $\boldsymbol{M}^{r,0}_{i} \gets \boldsymbol{\bar{M}}^{r}$;
		\FOR{$k = 1, \dots, K$}
		\STATE $\boldsymbol{G}^{r,k}_i\gets\nabla f_i(\boldsymbol{x}_i^{r, k} ; \xi_i)$;
		\STATE $\boldsymbol{M}^{r,k}_i = \beta \boldsymbol{M}^{r,k}_i + \boldsymbol{G}^{r,k}_i$;
		\STATE $\boldsymbol{U}^{r,k}_i, \boldsymbol{\Sigma}^{r,k}_i, \boldsymbol{V}^{r,k}_i = \mathrm{SVD}(\boldsymbol{M}^{r,k}_i)$;
		\STATE$\boldsymbol{x}^{r,k+1}_i =\boldsymbol{x}^{r,k}_i - \eta_t \big[(1-\alpha)\boldsymbol{U}^{r,k}_i {\boldsymbol{V}^{r,k}_i}^{\top} +\alpha\boldsymbol{\Delta}_G^r \big]$;
		
		\ENDFOR
		\STATE Communicate $( \boldsymbol{x}^{r, K}_i-\boldsymbol{x}^{r, 0}_i,  \boldsymbol{M}^{r,k}_i )$ to Server;
		\ENDFOR
		\STATE $\boldsymbol{\Delta}_G^r=\frac{-1}{SK\eta} \sum_{i=1}^S (\boldsymbol{x}^{r, K}_i-\boldsymbol{x}^{r, 0}_i)$;
		\STATE $\boldsymbol{x}^{r+1} =\boldsymbol{x}^{r} +\frac{1}{S} \sum_{i=1}^S (\boldsymbol{x}^{r, K}_i-\boldsymbol{x}^{r, 0}_i)$;
		\STATE $\boldsymbol{\bar{M}}^{r+1} =\frac{1}{S} \sum_{i=1}^S  \boldsymbol{M}^{r,k}_i$;
		\STATE Communicate $(\boldsymbol{x}^{r+1}, \boldsymbol{\bar{M}}^{r+1},\boldsymbol{\Delta}_G^r ) $ to Clients.
		\ENDFOR
	\end{algorithmic}
	\label{algorithm_fedadamw}
\end{algorithm}

To simplify the analysis, we consider the iterative rules  as in Algorithm 3, where we let $\beta_1=0$, $\lambda=0$. The local update takes the following rule: $$\boldsymbol{x}^{r,k+1}_i =\boldsymbol{x}^{r,k}_i - \eta_t \big[(1-\alpha)\boldsymbol{U}^{r,k}_i {\boldsymbol{V}^{r,k}_i}^{\top} + \lambda \boldsymbol{x}^{r,k}_i+\alpha\boldsymbol{\Delta}_G^r \big]$$ Algorithm 2 is the algorithm that performs better in practical situations, while Algorithm 3 is the algorithm used  for theoretical analysis. Algorithm 2 is the algorithm that performs better
in practical situations.

\section{Assumption}

We analyze generalization based on  following assumptions: 
\begin{assumption}
	\label{asp:smooth} \textit{(Smoothness).  $F_i$ is $L$-smooth for all $i \in$ $[N]$, 
		\begin{equation}
			\left\|\nabla F_i(\boldsymbol{\theta}_{1})-\nabla F_i(\boldsymbol{\theta}_{2})\right\| \leq L\|\boldsymbol{\theta}_{1}-\boldsymbol{\theta}_{2}\|
		\end{equation}
		for all $\boldsymbol{\theta}_{1}, \boldsymbol{\theta}_{2}$ in its domain and $i \in[N]$.}
\end{assumption}
\begin{assumption}\label{asp:data_var}  \textit{(Bounded variance of data heterogeneity). The global variability of the local gradient of the loss function is bounded by $\sigma_g^2$ for all $i \in[N]$, 
		\begin{equation}
			\left\|\nabla F_i\left(\boldsymbol{\theta}\right)-\nabla F\left(\boldsymbol{\theta}\right)\right\|^2 \leq \sigma_g^2
		\end{equation}
	}
\end{assumption}
\begin{assumption}
	(Bounded variance of stochastic gradient).\label{asp:sgd_var}  The stochastic gradient $\nabla F_i\left(\boldsymbol{\theta}, \xi_i\right)$, computed by the $i$-th client of model parameter $\boldsymbol{\theta}$ using mini-batch $\xi_i$, is an unbiased estimator of $\nabla F_i(\boldsymbol{\theta})$ with variance bounded by $\sigma_l^2$, i.e.,
	\begin{equation}
		\mathbb{E}_{\xi_i}\left\|\nabla F_i\left(\boldsymbol{\theta}, \xi_i\right) - \nabla F_i(\boldsymbol{\theta})\right\|^2 \leq \sigma_l^2
	\end{equation}
	for all $i \in [N]$, where the expectation is over all local datasets.
\end{assumption}

\renewcommand{\theassumption}{A.\arabic{assumption}}  
\setcounter{assumption}{0}  

\begin{assumption}[Smoothness]
\label{smoothness_appendix}
 \textit{The non-convex $f_{i}$ is one $L$-smooth function for all $i\in[m]$, i.e., $\Vert\nabla f_{i}(\boldsymbol{x})-\nabla f_{i}(\boldsymbol{y})\Vert\leq L\Vert\boldsymbol{x}-\boldsymbol{y}\Vert$, for all $\boldsymbol{x},\boldsymbol{y}\in\mathbb{R}^{d}$.}
\end{assumption}

\begin{assumption}[Bounded Stochastic Gradient I]
\label{bounded_stochastic_gradient_I_appendix}
\textit{$\boldsymbol{g}_{i}^{r}=\nabla f_{i}(\boldsymbol{x}_{i}^{r}, \xi_i^{r})$ computed by using a sampled mini-batch $\xi_i^{r}$ in client $i$ is an unbiased estimator of $\nabla f_{i}$ with bounded variance: $\mathbb{E}_{\xi_i^{r}}[\boldsymbol{g}_{i}^{r}]=\nabla f_{i}(\boldsymbol{x}_{i}^{r})$ and $\mathbb{E}_{\xi_i^{r}}\Vert g_{i}^{r} - \nabla f_{i}(\boldsymbol{x}_{i}^{r})\Vert^{2} \leq \sigma_{l}^{2}$.}
\end{assumption}

\begin{assumption}[Bounded Stochastic Gradient II]
\label{bounded_stochastic_gradient_II_appendix}
\textit{Each element of stochastic gradient $\boldsymbol{g}_{i}^{r}$ is bounded, i.e., $\Vert\boldsymbol{g}_{i}^{r}\Vert_{\infty}=\Vert f_{i}(\boldsymbol{x}_{i}^{r},\xi_i^{r})\Vert_{\infty}\leq G_{g}$, for all $\boldsymbol{x}_{i}^{r}\in\mathbb{R}^{d}$ and any mini-batch $\xi_i^{r}$.}
\end{assumption}

\begin{assumption}[Bounded Heterogeneity]
\label{bounded_heterogeneity_appendix}
\textit{The gradient dissimilarity between clients is bounded: $\Vert\nabla f_{i}(\boldsymbol{x})-\nabla f(\boldsymbol{x})\Vert^{2}\leq\sigma_{g}^{2}$, for all $\boldsymbol{x}\in\mathbb{R}^{d}$.}
\end{assumption}

In this section, we give the theoretical analysis of our proposed FedAdamW algorithm. Firstly we state some standard assumptions for the non-convex function $F$.

\subsection{ Main Lemmas} 
\begin{lemma} \label{lem:bias-var} Suppose $\left\{X_1, \cdots, X_\tau\right\} \subset \mathbb{R}^d$ be random variables that are potentially dependent. If their marginal means and variances satisfy $\mathbb{E}\left[X_i\right]=\mu_i$ and $\mathbb{E}\left[\| X_i-\right.$ $\left.\mu_i \|^2\right] \leq \sigma^2$, then it holds that
	$$\mathbb{E}\left[\left\|\sum_{i=1}^\tau X_i\right\|^2\right] \leq\left\|\sum_{i=1}^\tau \mu_i\right\|^2+\tau^2 \sigma^2.
	$$
	If they are correlated in the Markov way such that $\mathbb{E}\left[X_i \mid X_{i-1}, \cdots X_1\right]=\mu_i$ and $\mathbb{E}\left[\left\|X_i-\mu_i\right\|^2 \mid\right.$ $\left.\mu_i\right] \leq \sigma^2$, i.e., the variables $\left\{X_i-\mu_i\right\}$ form a martingale. Then the following tighter bound holds:
	$$
	\mathbb{E}\left[\left\|\sum_{i=1}^\tau X_i\right\|^2\right] \leq 2 \mathbb{E}\left[\left\|\sum_{i=1}^\tau \mu_i\right\|^2\right]+2 \tau \sigma^2.
	$$
\end{lemma}

\begin{lemma} \label{lem:par_sample} Given vectors $v_1, \cdots, v_N \in \mathbb{R}^d$ and $\bar{v}=\frac{1}{N} \sum_{i=1}^N v_i$, if we sample $\mathcal{S} \subset\{1, \cdots, N\}$ uniformly randomly such that $|\mathcal{S}|=S$, then it holds that
	
	$$
	\mathbb{E}\left[\left\|\frac{1}{S} \sum_{i \in \mathcal{S}} v_i\right\|^2\right]=\|\bar{v}\|^2+\frac{N-S}{S(N-1)} \frac{1}{N} \sum_{i=1}^N\left\|v_i-\bar{v}\right\|^2 .
	$$
\end{lemma}

\begin{proof}
	Letting $\mathbb{I}\{i \in \mathcal{S}\}$ be the indicator for the event $i \in \mathcal{S}_r$, we prove this lemma by direct calculation as follows:
	$$
	\begin{aligned}
		\mathbb{E}\left[\left\|\frac{1}{S} \sum_{i \in \mathcal{S}} v_i\right\|^2\right] & =\mathbb{E}\left[\left\|\frac{1}{S} \sum_{i=1}^N v_i \mathbb{I}\{i \in \mathcal{S}\}\right\|^2\right] \\
		& =\frac{1}{S^2} \mathbb{E}\left[\left(\sum_i\left\|v_i\right\|^2 \mathbb{I}\{i \in \mathcal{S}\}+2 \sum_{i<j} v_i^{\top} v_j \mathbb{I}\{i, j \in \mathcal{S}\}\right)\right] \\
		& =\frac{1}{S N} \sum_{i=1}^N\left\|v_i\right\|^2+\frac{1}{S^2} \frac{S(S-1)}{N(N-1)} 2 \sum_{i<j} v_i^{\top} v_j \\
		& =\frac{1}{S N} \sum_{i=1}^N\left\|v_i\right\|^2+\frac{1}{S^2} \frac{S(S-1)}{N(N-1)}\left(\left\|\sum_{i=1}^N v_i\right\|^2-\sum_{i=1}^N\left\|v_i\right\|^2\right) \\
		& =\frac{N-S}{S(N-1)} \frac{1}{N} \sum_{i=1}^N\left\|v_i\right\|^2+\frac{N(S-1)}{S(N-1)}\|\bar{v}\|^2 \\
		& =\frac{N-S}{S(N-1)} \frac{1}{N} \sum_{i=1}^N\left\|v_i-\bar{v}\right\|^2+\|\bar{v}\|^2 .
	\end{aligned}
	$$
\end{proof}

\section{Appendix A: Basic Assumptions and Notations}
Let $\mathcal{F}^0=\emptyset$ and $\mathcal{F}_i^{r, k}:=\sigma\left(\left\{x_i^{r, j}\right\}_{0 \leq j \leq k} \cup \mathcal{F}^r\right)$ and $\mathcal{F}^{r+1}:=\sigma\left(\cup_i \mathcal{F}_i^{r, K}\right)$ for all $r \geq 0$ where $\sigma(\cdot)$ indicates the $\sigma$-algebra. Let $\mathbb{E}_r[\cdot]:=\overline{\mathbb{E}}\left[\cdot \mid \mathcal{F}^r\right]$ be the expectation, conditioned on the filtration $\mathcal{F}^r$, with respect to the random variables $\left\{\mathcal{S}^r,\left\{\xi_i^{r, k}\right\}_{1 \leq i \leq N, 0 \leq k<K}\right\}$ in the $r$-th iteration. We also use $\mathbb{E}[\cdot]$ to denote the global expectation over all randomness in algorithms. Through out the proofs, we use $\sum_i$ to represent the sum over $i \in\{1, \ldots, N\}$, while $\sum_{i \in \mathcal{S}^r}$ denotes the sum over $i \in \mathcal{S}^r$. Similarly, we use $\sum_k$ to represent the sum of $k \in\{0, \ldots, K-1\}$. For all $r \geq 0$, we define the following auxiliary variables to facilitate proofs:

$$
\begin{aligned}
	\mathcal{E}_r & :=\mathbb{E}\left[\left\|\nabla f\left(x^{r}\right)-g^{r+1}\right\|^2\right] \\
	U_r & :=\frac{1}{N K} \sum_i \sum_k \mathbb{E}\left[\left\|x_i^{r, k}-x^{r}\right\|\right]^2 \\
	\zeta_i^{r, k} & :=\mathbb{E}\left[x_i^{r, k+1}-x_i^{r, k} \mid \mathcal{F}_i^{r, k}\right] \\
	\Xi_r & :=\frac{1}{N} \sum_{i=1}^N \mathbb{E}\left[\left\|\zeta_i^{r, 0}\right\|^2\right] \\
\end{aligned}
$$

Throughout the appendix, we let $\Delta:=f\left(x^0\right)-f^{\star}, G_0:=\frac{1}{N} \sum_i\left\|\nabla f_i\left(x^0\right)\right\|^2, x^{-1}:=x^0$ and $\mathcal{E}_{-1}:=$ $\mathbb{E}\left[\left\|\nabla f\left(x^0\right)-g^0\right\|^2\right]$. We will use the following foundational lemma for all our algorithms.

\section{FedMuon Algorithm Analyze}
\begin{lemma}\label{lem:Fedavg_grad_err}
	Under Assumption \ref{asp:smooth} , if $\gamma L \leq \frac{1}{24}$, the following holds all $r \geq 0$ :
	
	$$
	\mathbb{E}\left[f\left(x^{r+1}\right)\right] \leq \mathbb{E}\left[f\left(x^r\right)\right]-\frac{11 \gamma}{24} \mathbb{E}\left[\left\|\nabla f\left(x^r\right)\right\|^2\right]+\frac{13 \gamma}{24} \mathcal{E}_r
	$$
\end{lemma}

\begin{proof}
	Since $f$ is $L$-smooth, we have
	
	\begin{align}
		f(x^{r+1})
		&\le f(x^r) + \left\langle \nabla f(x^r),\, x^{r+1}-x^r \right\rangle
		+ \frac{L}{2}\left\| x^{r+1}-x^r \right\|^2 \\
		&= f(x^r) - \gamma \left\langle \nabla f(x^r),\, g^{r+1} \right\rangle
		+ \frac{L\gamma^2}{2}\left\| g^{r+1} \right\|^2 \\
		&= f(x^r) - \gamma \left\|\nabla f(x^r)\right\|^2
		+ \gamma \left\langle \nabla f(x^r),\, \nabla f(x^r) - g^{r+1} \right\rangle
		+ \frac{L\gamma^2}{2}\left\| g^{r+1} \right\|^2 .
	\end{align}

	Since $x^{r+1}=x^r-\gamma g^{r+1}$, using Young's inequality, we further have:
	\begin{align}
		f\left(x^{r+1}\right)
		&\leq f\left(x^r\right)-\frac{\gamma}{2}\left\|\nabla f\left(x^r\right)\right\|^2+\frac{\gamma}{2}\left\|\nabla f\left(x^r\right)-g^{r+1}\right\|^2+L \gamma^2\left(\left\|\nabla f\left(x^r\right)\right\|^2+\left\|\nabla f\left(x^r\right)-g^{r+1}\right\|^2\right) \\
		& \leq f\left(x^r\right)-\frac{11 \gamma}{24}\left\|\nabla f\left(x^r\right)\right\|^2+\frac{13 \gamma}{24}\left\|\nabla f\left(x^r\right)-g^{r+1}\right\|^2
	\end{align}
	
	where the last inequality is due to $\gamma L \leq \frac{1}{24}$. Taking the global expectation completes the proof.
\end{proof}

\begin{align*}
	&\mathbb{E}\left\| \nabla f(x^r) - \frac{1}{SK}\sum_{i \in S^r}\sum_{k=1}^K U_{i}^{r,k} V_{i}^{r,k \top} \right\|^2 \\
	&\leq \mathbb{E}\left\| \nabla f(x^r) - \frac{1}{SK}\sum_{i \in S^r}\sum_{k=1}^K U_{i}^{r,k} V_{i}^{r,k \top}
	+ \frac{1}{SK}\sum_{i \in S^r}\sum_{k=1}^K g_i^{r,k}
	- \frac{1}{SK}\sum_{i \in S^r}\sum_{k=1}^K g_i^{r,k} \right\|^2 \\
	&\leq 2\mathbb{E}\left\| \nabla f(x^r) - \frac{1}{SK}\sum_{i \in S^r}\sum_{k=1}^K g_i^{r,k} \right\|^2
	+ 2\mathbb{E}\left\| \frac{1}{SK}\sum_{i=1}^N\sum_{k=1}^K U_{i}^{r,k} V_{i}^{r,k \top}
	- \frac{1}{SK}\sum_{i \in S^r}\sum_{k=1}^K g_i^{r,k} \right\|^2 \\
	&\leq 2L^2 U_r^2 + \frac{2\sigma_l^2}{SK}
	+ 2\mathbb{E}\left\| \frac{1}{SK}\sum_{i \in S^r}\sum_{k=1}^K U_{i}^{r,k} V_{i}^{r,k \top}
	- \frac{1}{SK}\sum_{i \in S^r}\sum_{k=1}^K g_i^{r,k} \right\|^2 \\
	&\leq 2L^2 U_r^2 + \frac{2\sigma_l^2}{SK}
	+ 2\mathbb{E}\left\| \frac{1}{SK}\sum_{i \in S^r}\sum_{k=1}^K U_{i}^{r,k} V_{i}^{r,k \top}
	- \frac{1}{SK}\sum_{i \in S^r}\sum_{k=1}^K U_i^{r,k} S_i^{r,k} V_i^{r,k\top} \right\|^2 \\
	&\leq 2L^2 U_r^2 + \frac{2\sigma_l^2}{SK}
	+ 2\mathbb{E}\left\| \frac{1}{SK}\sum_{i \in S^r}\sum_{k=1}^K U_i^{r,k}(I - S_i^{r,k})V_i^{r,k\top} \right\|^2.\\
	&\leq 2L^2 U_r^2 + \frac{2\sigma_l^2}{SK}
	+ \frac{2}{SK} \sum_{i \in S^r} \sum_{k=1}^K \sum_{j=1}^d (1 - \sigma_{i,k,j})^2 \\
	&\leq 2L^2 U_r^2 + \frac{2\sigma_l^2}{SK}
	+ 2(1 - \sigma)^2 d
\end{align*}

\begin{lemma}\label{lem:Fedavg_client_drift}
	If $\gamma L \leq \frac{\beta}{6}$, the following holds for $r \geq 1$ :
	
	$$
	\mathcal{E}_r \leq\left(1-\frac{8 \beta}{9}\right) \mathcal{E}_{r-1}+\frac{4 \gamma^2 L^2}{\beta} \mathbb{E}\left[\left\|\nabla f\left(x^{r-1}\right)\right\|^2\right]+\frac{2 \beta^2 \sigma_l^2}{S K  }+8\beta L^2 U_r+8\beta(1 - \sigma)^2 d
	$$
	Additionally, it holds for $r=0$ that
	$$
	\mathcal{E}_0 \leq(1-\beta) \mathcal{E}_{-1}+\frac{4 \beta^2 \sigma_l^2}{S K}+8 \beta L^2 U_0+8\beta(1 - \sigma)^2 d
	$$
\end{lemma}

\begin{proof}
	For $r>1$,
	$$
	\begin{aligned}
		\mathcal{E}_r= & \mathbb{E}\left[\left\|\frac{1}{S K} \sum_{i \in S^r} \sum_{k=1}^K \nabla f\left(x^{r}\right) -g^{r+1}\right\|^2\right] \\
		= & \mathbb{E}\left[\left\|(1-\beta)\left(\frac{1}{S K} \sum_{i \in S^r} \sum_{k=1}^K \nabla f\left(x^{r}\right) -g^r\right)+\beta\left(\frac{1}{S K} \sum_{i \in S^r} \sum_{k=1}^K \nabla f\left(x^{r}\right) -\frac{1}{S K} \sum_{i \in S^r} \sum_{k=1}^K U_{i}^{r,k} V_{i}^{r,k \top} \right)\right\|^2\right] \\
		\leq  & \mathbb{E}\left[\left\|(1-\beta)\left(\frac{1}{S K} \sum_{i \in S^r} \sum_{k=1}^K \nabla f\left(x^{r}\right) -g^r\right)\right\|^2\right]+\beta^2  \mathbb{E}\left[\left\|\nabla f\left(x^{r}\right)-\frac{1}{S K} \sum_{i \in S^r} \sum_{k=1}^K U_{i}^{r,k} V_{i}^{r,k \top}\right\|^2\right] \\
		& +2 \beta \mathbb{E}\left[\left\langle(1-\beta)\left(\frac{1}{S K} \sum_{i \in S^r} \sum_{k=1}^K \nabla f\left(x^{r}\right) -g^r\right), \frac{1}{S K} \sum_{i \in S^r} \sum_{k=1}^K \nabla f\left(x^{r}\right) -\frac{1}{S K} \sum_{i \in S^r} \sum_{k=1}^K U_{i}^{r,k} V_{i}^{r,k \top} \right\rangle\right] .
	\end{aligned}
	$$
	Note that $\left\{\nabla F\left(x_i^{r, k} ; \xi_i^{r, k}\right)\right\}_{0 \leq k<K}$ are sequentially correlated. Applying the AM-GM inequality and Lemma  \ref{lem:bias-var}, we have
	$$
	\mathcal{E}_r \leq\left(1+\frac{\beta}{2}\right) \mathbb{E}\left[\left\|(1-\beta)\left(\nabla f\left(x^{r}\right)-g^r\right)\right\|^2\right]+4\beta L^2 U_r+4\beta(1 - \sigma)^2 d+4 \beta^2\left(\frac{\sigma_l^2}{S K}+L^2 U_r+(1 - \sigma)^2 d\right)
	$$
	Using the AM-GM inequality again and Assumption \ref{asp:smooth}, we have
	$$
	\begin{aligned}
		\mathcal{E}_r & \leq(1-\beta)^2\left(1+\frac{\beta}{2}\right)\left[\left(1+\frac{\beta}{2}\right) \mathcal{E}_{r-1}+\left(1+\frac{2}{\beta}\right) L^2 \mathbb{E}\left[\left\|x^{r}-x^{r-1}\right\|^2\right]\right]+\frac{4 \beta^2 \sigma_l^2}{S K }+8 \beta L^2 U_r+8\beta(1 - \sigma)^2 d \\
		& \leq(1-\beta) \mathcal{E}_{r-1}+\frac{2}{\beta} L^2 \mathbb{E}\left[\left\|x^{r}-x^{r-1}\right\|^2\right]+\frac{4 \beta^2 \sigma_l^2}{S K }+8 \beta L^2 U_r+8\beta(1 - \sigma)^2 d \\
		& \leq\left(1-\frac{8 \beta}{9}\right) \mathcal{E}_{r-1}+4 \frac{\gamma^2 L^2}{\beta} \mathbb{E}\left[\left\|\nabla f\left(x^{r-1}\right)\right\|^2\right]+\frac{4 \beta^2 \sigma_l^2}{S K }+8 \beta L^2 U_r+8\beta(1 - \sigma)^2 d
	\end{aligned}
	$$
	
	where we plug in $\left\|x^{r}-x^{r-1}\right\|^2 \leq 2 \gamma^2\left(\left\|\nabla f\left(x^{r-1}\right)\right\|^2+\left\|g^r-\nabla f\left(x^{r-1}\right)\right\|^2\right)$ and use $\gamma L \leq \frac{\beta}{6}$ in the last inequality. Similarly for $r=0$,
	
	$$
	\begin{aligned}
		\mathcal{E}_0 & \leq\left(1+\frac{\beta}{2}\right) \mathbb{E}\left[\left\|(1-\beta)\left(\nabla f\left(x^0\right)-g^0\right)\right\|^2\right]+4 \beta L^2 U_0+4 \beta^2\left(\frac{\sigma_l^2}{S K}+L^2 U_0\right) \\
		& \leq(1-\beta) \mathcal{E}_{-1}+\frac{4 \beta^2  \sigma_l^2}{S K }+8\beta L^2 U_0+8\beta(1 - \sigma)^2 d
	\end{aligned}
	$$
\end{proof}

\begin{lemma}
	If $\eta L K \leq \frac{1}{\beta}$, the following holds for $r \geq 0$ :
	
	$$
	U_r \leq 2 e K^2 \Xi_r+K \eta^2 \beta^2   \sigma_l^2\left(1+2 K^3 L^2 \eta^2 \beta^2\right)
	$$
\end{lemma}
\begin{proof}
	Recall that $\zeta_i^{r, k}:=\mathbb{E}\left[x_i^{r, k+1}-x_i^{r, k} \mid \mathcal{F}_i^{r, k}\right]=-\eta\left((1-\beta) g^r+\beta \nabla f_i\left(x_i^{r, k}\right)\right)$. Then we have
	$$
	\begin{aligned}
		\mathbb{E}\left[\left\|\zeta_i^{r, j}-\zeta_i^{r, j-1}\right\|^2\right] & \leq  \eta^2 L^2 \beta^2 \mathbb{E}\left[\left\|x_i^{r, j}-x_i^{r, j-1}\right\|^2\right] \\
		& \leq  \eta^2 L^2 \beta^2\left(\eta^2 \beta^2 \sigma_l^2+\mathbb{E}\left[\left\|\zeta_i^{r, j-1}\right\|^2\right)\right.
	\end{aligned}
	$$
	For any $1 \leq j \leq k-1 \leq K-2$, using $\eta L \leq \frac{1}{\beta K} \leq \frac{1}{\beta(k+1)}$, we have
	$$
	\begin{aligned}
		\mathbb{E}\left[\left\|\zeta_i^{r, j}\right\|^2\right] & \leq\left(1+\frac{1}{k}\right) \mathbb{E}\left[\left\|\zeta_i^{r, j-1}\right\|^2\right]+(1+k) \mathbb{E}\left[\left\|\zeta_i^{r, j}-\zeta_i^{r, j-1}\right\|^2\right] \\
		& \leq\left(1+\frac{2}{k}\right) \mathbb{E}\left[\left\|\zeta_i^{r, j-1}\right\|^2\right]+(k+1)  L^2 \eta^4 \beta^4 \sigma_l^2 \\
		& \leq e^2 \mathbb{E}\left[\left\|\zeta_i^{r, 0}\right\|^2\right]+4  k^2 L^2 \eta^4 \beta^4 \sigma_l^2
	\end{aligned}
	$$
	where the last inequality is by unrolling the recursive bound and using $\left(1+\frac{2}{k}\right)^k \leq e^2$. By Lemma \ref{lem:bias-var} , it holds that for $k \geq 2$,
	$$
	\begin{aligned}
		\mathbb{E}\left[\left\|x_i^{r, k}-x^{r}\right\|^2\right] & \leq 2 \mathbb{E}\left[\left\|\sum_{j=0}^{k-1} \zeta_i^{r, j}\right\|^2\right]+2  k \eta^2 \beta^2 \sigma_l^2 \\
		& \leq 2 k \sum_{j=0}^{k-1} \mathbb{E}\left[\left\|\zeta_i^{r, k}\right\|^2\right]+2  k \eta^2 \beta^2 \sigma_l^2 \\
		& \leq 2 e^2 k^2 \mathbb{E}\left[\left\|\zeta_i^{r, 0}\right\|^2\right]+2  k \eta^2 \beta^2 \sigma_l^2\left(1+4 k^3 L^2 \eta^2 \beta^2\right)
	\end{aligned}
	$$
	This is also valid for $k=0,1$. Summing up over $i$ and $k$ finishes the proof.
\end{proof}

\begin{lemma}\label{lem:Fedavg_grad_norm}
	If $288 e(\eta K L)^2\left((1-\beta)^2+e(\beta \gamma L R)^2\right) \leq 1$, then it holds for $r \geq 0$ that
	
	$$
	\sum_{r=0}^{R-1} \Xi_r \leq \frac{1}{72 e K^2 L^2} \sum_{r=-1}^{R-2}\left(\mathcal{E}_r+\mathbb{E}\left[\left\|\nabla f\left(x^{r}\right)\right\|^2\right]\right)+2 \eta^2 \beta^2  e R G_0
	$$
\end{lemma}
\begin{proof}
	Note that $\zeta_i^{r, 0}=-\eta\left((1-\beta) g^r+\beta \nabla f_i\left(x^{r}\right)\right)$,
	$$
	\frac{1}{N} \sum_{i=1}^N\left\|\zeta_i^{r, 0}\right\|^2 \leq 2 \eta^2\left((1-\beta)^2\left\|g^r\right\|^2+\beta^2 \frac{1}{N} \sum_{i=1}^N\left\|\nabla f_i\left(x^{r}\right)\right\|^2\right)
	$$
	Using Young's inequality, we have for any $q>0$ that
	$$
	\begin{aligned}
		\mathbb{E}\left[\left\|\nabla f_i\left(x^{r}\right)\right\|^2\right] & \leq(1+q) \mathbb{E}\left[\left\|\nabla f_i\left(x^{r-1}\right)\right\|^2\right]+\left(1+q^{-1}\right) L^2 \mathbb{E}\left[\left\|x^{r}-x^{r-1}\right\|^2\right] \\
		& \leq(1+q) \mathbb{E}\left[\left\|\nabla f_i\left(x^{r-1}\right)\right\|^2\right]+2\left(1+q^{-1}\right) \gamma^2 L^2\left(\mathcal{E}_{r-1}+\mathbb{E}\left[\left\|\nabla f\left(x^{r-1}\right)\right\|^2\right]\right) \\
		& \leq(1+q)^r \mathbb{E}\left[\left\|\nabla f_i\left(x^0\right)\right\|^2\right]+\frac{2}{q} \gamma^2 L^2 \sum_{j=0}^{r-1}\left(\mathcal{E}_j+\mathbb{E}\left[\left\|\nabla f\left(x^j\right)\right\|^2\right)(1+q)^{r-j}\right.
	\end{aligned}
	$$
	Take $q=\frac{1}{r}$ and we have
	
	\begin{equation}\label{eqn:vninbvsdvds}
		\mathbb{E}\left[\left\|\nabla f_i\left(x^{r}\right)\right\|^2\right] \leq e \mathbb{E}\left[\left\|\nabla f_i\left(x^0\right)\right\|^2\right]+2 e(r+1) \gamma^2 L^2 \sum_{j=0}^{r-1}\left(\mathcal{E}_j+\mathbb{E}\left[\left\|\nabla f\left(x^j\right)\right\|^2\right)\right.
	\end{equation}
	Note that this inequality is valid for $r=0$. Therefore, using \eqref{eqn:vninbvsdvds}, we have
	$$
	\begin{aligned}
		\sum_{r=0}^{R-1} \Xi_r \leq & \sum_{r=0}^{R-1} 2 \eta^2 \mathbb{E}\left[(1-\beta)^2\left\|g^r\right\|^2+\beta^2  \frac{1}{N} \sum_{i=1}^N\left\|\nabla f_i\left(x^{r}\right)\right\|^2\right] \\
		\leq & \sum_{r=0}^{R-1} 2 \eta^2\left(2(1-\beta)^2\left(\mathcal{E}_{r-1}+\mathbb{E}\left[\left\|\nabla f\left(x^{r-1}\right)\right\|^2\right]\right)+\beta^2  \frac{1}{N} \sum_{i=1}^N \mathbb{E}\left[\left\|\nabla f_i\left(x^{r}\right)\right\|^2\right]\right) \\
		\leq & \sum_{r=0}^{R-1} 4 \eta^2(1-\beta)^2\left(\mathcal{E}_{r-1}+\mathbb{E}\left[\left\|\nabla f\left(x^{r-1}\right)\right\|^2\right]\right) \\
		& +2 \eta^2 \beta^2  \sum_{r=0}^{R-1}\left(\frac{e}{N} \sum_{i=1}^N \mathbb{E}\left[\left\|\nabla f_i\left(x^0\right)\right\|^2\right]+2 e(r+1)(\gamma L)^2 \sum_{j=0}^{r-1}\left(\mathcal{E}_j+\mathbb{E}\left[\left\|\nabla f\left(x^j\right)\right\|^2\right]\right)\right) \\
		\leq & 4 \eta^2(1-\beta)^2 \sum_{r=0}^{R-1}\left(\mathcal{E}_{r-1}+\mathbb{E}\left[\left\|\nabla f\left(x^{r-1}\right)\right\|^2\right]\right) \\
		& +2 \eta^2 \beta^2 \left(e R G_0+2 e(\gamma L R)^2 \sum_{r=0}^{R-2}\left(\mathcal{E}_r+\mathbb{E}\left[\left\|\nabla f\left(x^{r}\right)\right\|^2\right]\right)\right)
	\end{aligned}
	$$
	Rearranging the equation and applying the upper bound of $\eta$ completes the proof.
\end{proof}
\begin{theorem}[Convergence for non-convex functions]	
	Under Assumption \ref{asp:smooth} and \ref{asp:sgd_var} , if we take $g^0=0$,
	$$
	\begin{aligned}
		& \beta=\min \left\{, \sqrt{\frac{S K L \Delta}{\sigma_l^2 R}}\right\} \text { for any constant } c \in(0,1], \quad \gamma=\min \left\{\frac{1}{24 L}, \frac{\beta}{6 L}\right\}, \\
		& \eta K L \lesssim \min \left\{1, \frac{1}{\beta \gamma L R},\left(\frac{L \Delta}{G_0 \beta^3 R}\right)^{1 / 2}, \frac{1}{(\beta N)^{1 / 2}}, \frac{1}{\left(\beta^3 N K\right)^{1 / 4}}\right\}
	\end{aligned}
	$$
	then DP-FedPGN converges as
	
	$$
	\frac{1}{R} \sum_{r=0}^{R-1} \mathbb{E}\left[\left\|\nabla f\left(x^{r}\right)\right\|^2\right] \lesssim \sqrt{\frac{L \Delta \sigma_l^2}{S K R}}+\frac{L \Delta}{R} .
	$$

	Here $G_0:=\frac{1}{N} \sum_{i=1}^N\left\|\nabla f_i\left(x^0\right)\right\|^2$.
\end{theorem}
\begin{proof}
	Combining Lemma \ref{lem:Fedavg_grad_err} and \ref{lem:Fedavg_client_drift}, we have
	
	$$
	\begin{aligned}
		\mathcal{E}_r \leq & \left(1-\frac{8 \beta}{9}\right) \mathcal{E}_{r-1}+4 \frac{(\gamma L)^2}{\beta} \mathbb{E}\left[\left\|\nabla f\left(x^{r-1}\right)\right\|^2\right]+\frac{4 \beta^2 \sigma_l^2}{S K  } +8\beta(1 - \sigma)^2 d\\
		& +4 \beta L^2\left(2 e K^2 \Xi_r+K \eta^2 \beta^2   \sigma_l^2\left(1+2 K^3 L^2 \eta^2 \beta^2\right)\right.
	\end{aligned}
	$$
	and
	$$
	\mathcal{E}_0 \leq(1-\beta) \mathcal{E}_{-1}+\frac{4 \beta^2 \sigma_l^2}{S K  }+8\beta(1 - \sigma)^2 d+4 \beta   L^2\left(2 e K^2 \Xi_0+K \eta^2 \beta^2   \sigma_l^2\left(1+2 K^3 L^2 \eta^2 \beta^2\right)\right) .
	$$
	Summing over $r$ from 0 to $R-1$ and applying Lemma \ref{lem:Fedavg_grad_norm},
	$$
	\begin{aligned}
		\sum_{r=0}^{R-1} \mathcal{E}_r \leq & \left(1-\frac{8 \beta}{9}\right) \sum_{r=-1}^{R-2} \mathcal{E}_r+4 \frac{(\gamma L)^2}{\beta} \sum_{r=0}^{R-2} \mathbb{E}\left[\left\|\nabla f\left(x^{r}\right)\right\|^2\right]+4 \frac{\beta^2 \sigma_l^2}{S K  } R +8\beta(1 - \sigma)^2 dR\\
		& +4 \beta  L^2\left(2 e K^2 \sum_{r=0}^{R-1} \Xi_r+R K \eta^2 \beta^2  \sigma_l^2\left(1+2 K^3 L^2 \eta^2 \beta^2\right)\right) \\
		\leq & \left(1-\frac{7 \beta}{9}\right) \sum_{r=-1}^{R-2} \mathcal{E}_r+\left(4 \frac{(\gamma L)^2}{\beta}+\frac{\beta}{9}\right) \sum_{r=-1}^{R-2} \mathbb{E}\left[\left\|\nabla f\left(x^{r}\right)\right\|^2\right]+16 \beta^3  (e \eta K L)^2 R G_0 \\
		& +\frac{4 \beta^2 \sigma_l^2}{S K  } R+8\beta(1 - \sigma)^2 dR+4 \beta^3 (\eta K L)^2\left(\frac{1}{K}+2(\eta K L \beta)^2\right) \sigma_l^2 R \\
		\leq & \left(1-\frac{7 \beta}{9}\right) \sum_{r=-1}^{R-2} \mathcal{E}_r+\frac{2 \beta}{9} \sum_{r=-1}^{R-2} \mathbb{E}\left[\left\|\nabla f\left(x^{r}\right)\right\|^2\right]+16 \beta^3  (e \eta K L)^2 R G_0+\frac{8 \beta^2 \sigma_l^2}{S K  } R+8\beta(1 - \sigma)^2 dR
	\end{aligned}
	$$
	Here in the last inequality we apply
	$$
	4 \beta  (\eta K L)^2\left(\frac{1}{K}+2(\eta K L \beta)^2\right) \leq \frac{2}{N K} \quad \text { and } \quad \gamma L \leq \frac{\beta}{6} .
	$$
	Therefore,
	$$
	\sum_{r=0}^{R-1} \mathcal{E}_r \leq \frac{9}{7 \beta} \mathcal{E}_{-1}+\frac{2}{7} \mathbb{E}\left[\sum_{r=-1}^{R-2}\left\|\nabla f\left(x^{r}\right)\right\|^2\right]+\frac{144}{7}(e \beta \eta K L)^2 G_0 R+\frac{36 \beta \sigma_l^2}{7 S K} R +\frac{72}{7}(1 - \sigma)^2 dR.
	$$
	Combine this inequality with Lemma \ref{lem:Fedavg_grad_err} and we get
	$$
	\frac{1}{\gamma} \mathbb{E}\left[f\left(x^{r}\right)-f\left(x^0\right)\right] \leq-\frac{1}{7} \sum_{r=0}^{R-1} \mathbb{E}\left[\left\|\nabla f\left(x^{r}\right)\right\|^2\right]+\frac{39}{56 \beta} \mathcal{E}_{-1}+\frac{78}{7}(e \beta \eta K L)^2 G_0 R+\frac{39 \beta \sigma_l^2}{14 S K} R+\frac{72}{7}(1 - \sigma)^2 dR. 
	$$
	Finally, noticing that $g^0=0$ implies $\mathcal{E}_{-1} \leq 2 L\left(f\left(x^0\right)-f^*\right)=2 L \Delta$, we obtain
	$$
	\begin{aligned}
		\frac{1}{R} \sum_{r=0}^{R-1} \mathbb{E}\left[\left\|\nabla f\left(x^{r}\right)\right\|^2\right] & \lesssim \frac{L \Delta}{\gamma L R}+\frac{\mathcal{E}_{-1}}{\beta R}+(\beta \eta K L)^2 G_0+\frac{\beta \sigma_l^2}{S K}+(1 - \sigma)^2 d. \\
		& \lesssim \frac{L \Delta}{R}+\frac{L \Delta}{\beta R}+\frac{\beta \sigma_l^2}{S K}+(\beta \eta K L)^2 G_0 \\
		& \lesssim \frac{L \Delta}{R}+\sqrt{\frac{L \Delta \sigma_l^2}{S K R}}
	\end{aligned}
	$$
\end{proof}

\end{document}